\DeclareTextFontCommand{\bemph}{\bfseries\em}
\documentclass{article}

\usepackage{amsmath,amsthm,amssymb,mathtools}
\newtheorem{theorem}{{Theorem}}
\newtheorem{lemma}{{Lemma}}

\usepackage{bm}
\usepackage{xr}
\usepackage{commath} %
\usepackage{multirow}

\usepackage[english]{babel}
\usepackage[parfill]{parskip}
\usepackage{afterpage}
\usepackage{enumitem}
\usepackage{framed}
\usepackage{xspace}
\usepackage{lineno}

\usepackage{ragged2e}

\newcounter{parcount}

\usepackage[usenames,dvipsnames]{xcolor}
\definecolor{shadecolor}{gray}{0.9}

\usepackage{graphicx}
\usepackage[labelfont=bf]{caption}
\usepackage[format=hang]{subcaption}

\usepackage{listings}
\usepackage{fancyvrb}
\fvset{fontsize=\small}

\usepackage{booktabs,tabularx,siunitx}
\usepackage{etoolbox}
\renewrobustcmd{\bfseries}{\fontseries{b}\selectfont}
\renewrobustcmd{\boldmath}{}
\newrobustcmd{\BF}{\bfseries}
\newcommand{\bx}{\mathbf{x}}

\newcommand{\calJ}{{\mathcal{J}}}

\newcommand{\calL}{{\mathcal{L}}}

\newcommand{\calX}{{\mathcal{X}}}
\newcommand{\calY}{{\mathcal{Y}}}
\newcommand{\calZ}{{\mathcal{Z}}}

\newcommand{\Real}{\mathbb R}

\usepackage{color}
\definecolor{orange}{rgb}{1,0.3,0}
\definecolor{copper}{rgb}{1,.62,.40}
\definecolor{hotpink}{rgb}{1,0,0.5}
\definecolor{darkgreen1}{rgb}{0, .35, 0}
\definecolor{darkgreen}{rgb}{0, .6, 0}
\definecolor{darkred}{rgb}{.75,0,0}

\newcommand{\be}{\begin{eqnarray}}
\newcommand{\ee}{\end{eqnarray}}
\newcommand{\bee}{\begin{eqnarray*}}
\newcommand{\eee}{\end{eqnarray*}}

\newcommand{\matrixb}{\left[ \begin{array}}
\newcommand{\matrixe}{\end{array} \right]}

\newcommand{\argmax}{\operatornamewithlimits{\arg \max}}
\newcommand{\argmin}{\operatornamewithlimits{\arg \min}}

\newcommand{\E}{\mathbb E}

\newcommand{\triloss}{\calL_{\text{triplet}}}
\DeclarePairedDelimiterX{\infdivx}[2]{(}{)}{%
  #1\;\delimsize\|\;#2%
}
\newcommand{\kld}[2]{\ensuremath{D_{KL}\infdivx{#1}{#2}}\xspace}

\newcommand{\x}{\mathbf{x}} %
\newcommand{\tx}{\mathbf{\widetilde{x}}} %
\newcommand{\txval}{\widetilde{x}} %
\newcommand{\y}{\mathbf{y}} %
\newcommand{\tX}{\widetilde{X}} %
\newcommand{\DIMCO}{\textbf{D}iscrete \textbf{I}nfo\textbf{M}ax \textbf{CO}des }

\usepackage[accepted]{preamble/icml2020}

\usepackage[
    linktoc=all,
    bookmarks=false,
    hypertexnames=true,
    colorlinks=true, 
    urlcolor=BrickRed!90!Sepia,
    linkcolor=BrickRed!90!Sepia,
    citecolor=Aquamarine!30!Blue]{hyperref}
\usepackage[capitalize,nameinlink,noabbrev]{cleveref} %
\creflabelformat{equation}{#2\textup{#1}#3}  %

\begin{document}
\icmltitlerunning{Discrete Infomax Codes for Supervised Representation Learning}
\twocolumn[
\icmltitle{Discrete Infomax Codes for Supervised Representation Learning}

\icmlsetsymbol{equal}{*}

\begin{icmlauthorlist}
\icmlauthor{Yoonho Lee}{kakao}
\icmlauthor{Wonjae Kim}{kakao}
\icmlauthor{Wonpyo Park}{kakao}
\icmlauthor{Seungjin Choi}{baro}
\end{icmlauthorlist}

\icmlaffiliation{kakao}{Kakao Corporation, Korea}
\icmlaffiliation{baro}{BARO, Korea}

\icmlkeywords{Machine Learning, ICML}
\icmlcorrespondingauthor{Yoonho Lee}{einet89@gmail.com}

\vskip 0.3in
]

\printAffiliationsAndNotice{}  %

\begin{abstract}
Learning compact discrete representations of data is a key task on its own or for facilitating subsequent processing of data. 
In this paper we present a model that produces \DIMCO (DIMCO);
we learn a probabilistic encoder that yields $k$-way $d$-dimensional codes associated with input data.
Our model's learning objective is to maximize the mutual information between codes and labels with a regularization, which enforces entries of a codeword to be as independent as possible.
We show that the infomax principle also justifies previous loss functions (e.g., cross-entropy) as its special cases. 
Our analysis also shows that using shorter codes, as DIMCO does, reduces overfitting in the context of few-shot classification.
Through experiments in various domains, we observe this implicit meta-regularization effect of DIMCO.
Furthermore, we show that the codes learned by DIMCO are efficient in terms of both memory and retrieval time compared to previous methods.

\end{abstract}

\section{Introduction} \label{sec:intro}
Metric learning and few-shot classification are two problem setups that test a model's ability to classify data from classes that were unseen during training.
Such problems are also commonly interpreted as testing meta-learning ability,
since the process of constructing a classifier with examples from new classes can be seen as learning.
Many recent works \citep{hoffer2015deep,movshovitz2017no,snell2017prototypical,oreshkin2018tadam} 
tackle this problem by learning a continuous embedding ($\tx \in \Real^n$) of datapoints.
Such models compare pairs of embeddings using e.g., Euclidean distance to perform nearest neighbor classification.
However, it remains unclear whether such models effectively utilize the entire space of $\Real^n$.

Information theory provides a framework in which we can effectively ask such questions about representation schemes.
In particular, the \bemph{information bottleneck} principle \citep{tishby2000information,shwartz2017opening} formalizes the optimality of a representation.
This principle states that the optimal representation $\tX$ is one that maximally compresses the input $X$ while also being predictive of labels $Y$.
From this viewpoint, we see that the previous methods which map data to $\Real^n$ focus on being predictive of labels $Y$ but not on compressing $X$.

The degree of compression of an embedding is the number of bits it reflects about the original data.
Note that for continuous embeddings, each of the $n$ numbers in a $n$-dimensional embedding requires $32$ bits; 
it is unlikely that unconstrained optimization of such embeddings use all of these $32n$ bits effectively.
We propose to resolve this limitation by instead using \bemph{discrete embeddings} and controlling the number of bits in each dimension via hyperparameters.
To this end, we propose a model that produces \DIMCO (DIMCO) via an end-to-end learnable neural network encoder.

This work's primary contributions are as follows.
We motivate mutual information as an objective for learning embeddings,
and propose an efficient method of estimating it in the discrete case.
We experimentally demonstrate that learned discrete embeddings are more memory- and time- efficient compared to continuous embeddings.
Our experiments also show that using discrete embeddings helps meta-generalization by acting as an information bottleneck.
We also provide theoretical support for this connection through an information-theoretic PAC bound that shows the generalization characteristics of learned discrete codes.

This paper is organized as follows.
We propose our model for learning discrete codes in \cref{sec:method}.
We justify our loss function and also provide generalization bound for our setup in \cref{sec:analysis}.
In \cref{sec:experiments}, we present experiments that 
We compare our method to related work in \cref{sec:previous} and conclude our paper in \cref{sec:discussion}.

\section{Discrete Infomax Codes (DIMCO)} \label{sec:method}
\begin{figure*}
\centering
\begin{subfigure}[b]{0.55\textwidth}
\includegraphics[width=\linewidth]{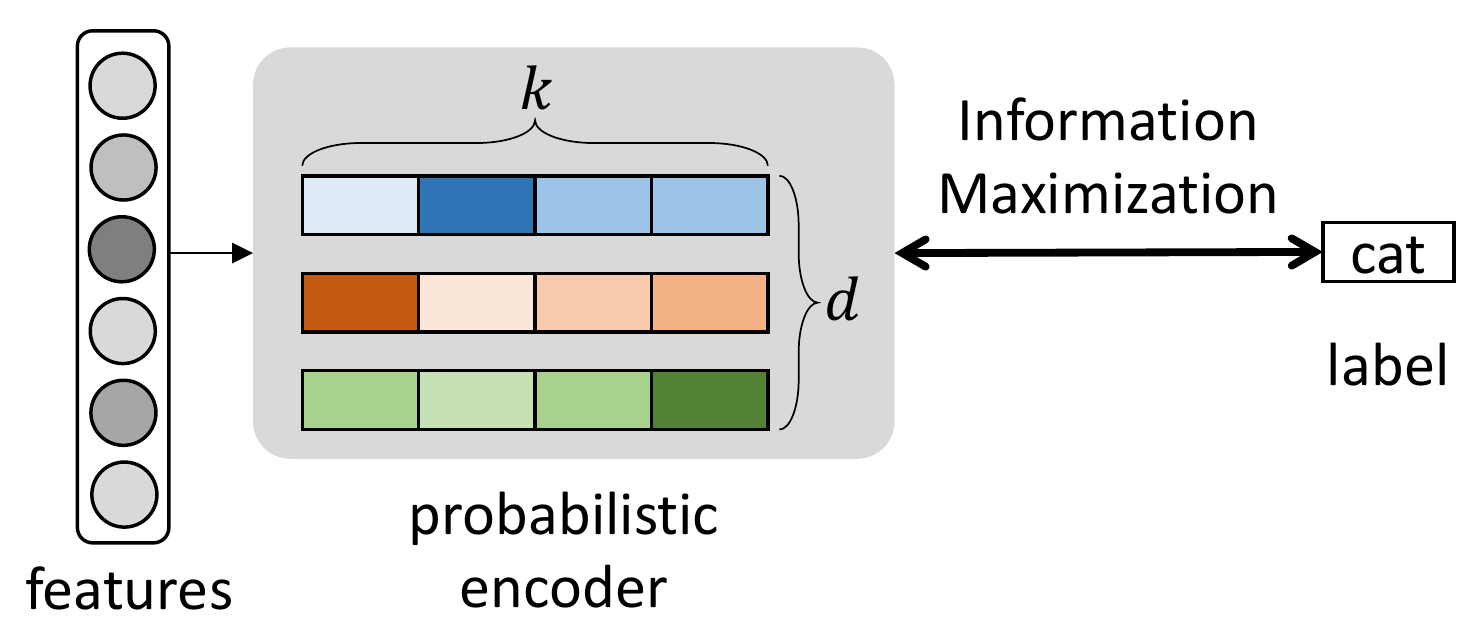} 
\caption{Training objective}
\end{subfigure}
\hspace{1cm}
\begin{subfigure}[b]{0.35\textwidth}
\includegraphics[width=\linewidth]{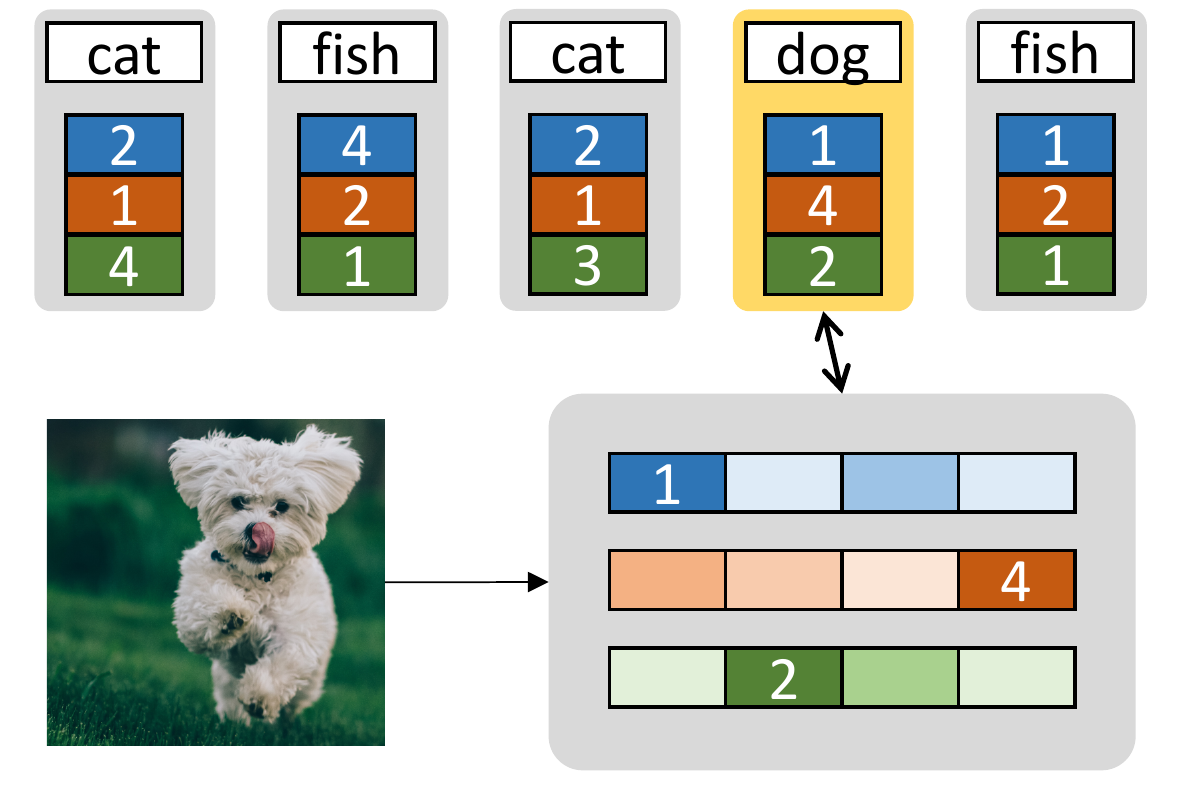}
\caption{Retrieval}
\end{subfigure}
\vspace{-5pt}
\caption{
    A graphical overview of \DIMCO (DIMCO).
    (a) Discrete codes are produced by a probabilistic encoder that maps each datapoint to a distribution of $k$-way $d$-dimensional discrete codes.
    The encoder is trained to maximize the mutual information between code distribution and label distribution.
    (b) We classify query images by comparing against a support set of discrete codes and corresponding labels.
    Our similarity metric is the query's log probability of each discrete code.
}
\label{fig:overview}
\vspace{-5pt}
\end{figure*}

We present our model which produces \DIMCO (DIMCO).
A deep neural network is trained end-to-end to learn $k$-way $d$-dimensional discrete codes that maximally preserves the information on labels.
We outline the training procedure in \cref{alg:dimco}, 
and we also illustrate the overall structure in the case of $4$-way $3$-dimensional codes ($k=4, d=3$) in \cref{fig:overview}.

\subsection{Learnable Discrete Codes}
\label{subsec:net}

Suppose that we are given a set of labeled examples, which are realizations of random variables $(X, Y) \sim p(\bx,y)$,
where $X$ is the input, and its corresponding label is $Y$. 
We denote its codebook by $\tX$, which is a compressed representation of $X$.
Realizations of $X$ and $Y$ are denoted by $\bx \in \Real^D$ and $y \in \{1,\ldots,c\}$.

We construct a probabilistic encoder $p(\tx | \bx)$, which is implemented by a deep neural network, 
that maps an input $\bx$ to a \bemph{$\bm{k}$-way $\bm{d}$-dimensional code} $\tx \in \{1,2,\ldots,k\}^d$.
That is, each entry of $\tx$ takes on one of $k$ possible values 
and the cardinality of $\tX$ is $|\tX|=k^d$.
Special cases of this coding scheme include $k$-way class labels ($d=1$), $d$-dimensional binary codes ($k=2$), and even fixed-length decimal integers ($k=10$).

We now describe our model which produces discrete infomax codes.
A neural network encoder $f(\bx; \theta)$ outputs $k$-dimensional categorical distributions, $\mathrm{Cat}(p_{i,1}, \ldots, p_{i,k})$. 
Here, $p_{i,j} (\bx)$ represents the probability that output variable $i$ takes on value $j$, consuming $\bx$ as an input, for $i=1,\ldots,d$ and $j=1,\ldots,k$.
The encoder takes $\bx$ as an input to produce logits $\{l_{i,j}\}$, which are reshaped into a $d \times k$ matrix:
\be
\label{eq:cont_rep}
\begin{bmatrix}
    l_{1,1} & l_{1,2} & l_{1,3} & \dots  & l_{1,k} \\
    l_{2,1} & l_{2,2} & l_{2,3} & \dots  & l_{2,k} \\
    \vdots  & \vdots  & \vdots  & \ddots & \vdots  \\
    l_{d,1} & l_{d,2} & l_{d,3} & \dots  & l_{d,k}
\end{bmatrix}.
\ee
These logits undergo softmax functions to yield
\begin{eqnarray}
\label{eq:to_prob}
p_{i,j} = \frac{\exp({l_{i,j}})}{\sum_{j=1}^{k} \exp({l_{i,j}})}.
\end{eqnarray}
Each example $\bx$ in the training set is assigned a codeword $\tx=[\widetilde{x}_1,\ldots, \widetilde{x}_d]^{\top}$, 
each entry of which is determined by one of $k$ events that is most probable, i.e.,
\begin{eqnarray}
\label{eq:argmax}
\tx 
= \Big[
    \overbrace{\argmax_j p_{1j}}^{\txval_1}, \hspace{5pt} 
    \cdots, \hspace{5pt} 
    \overbrace{\argmax_j p_{dj}}^{\txval_d} 
\Big]^{\top}.
\end{eqnarray}
While the stochastic encoder $p(\tx | \bx)$ indues a soft partitioning of input data, 
codewords assigned by the rule in (\ref{eq:argmax}) yields a hard partitioning of $X$.

\subsection{Loss Function}
\label{subsec:loss}

\begin{algorithm}[h]
\caption{DIMCO Training Procedure}
\label{alg:dimco}
\begin{algorithmic}
    \STATE Initialize network parameters $\theta$
    \REPEAT
    \STATE Sample a batch $\{(\bx^n, y^n)\}$
    \STATE Compute logits $\{ l_{i, j}^n \} = f(\bx^n; \theta)$ \hfill  via \eqref{eq:cont_rep}
    \STATE Compute probs $\{ p_{i, j}^n \} = \textrm{softmax}(\{ l_{i, j}^n \})$ \hfill  via \eqref{eq:to_prob}
    \STATE Update $\theta$ to minimize loss $\calL(\{p_{i,j}^n\}, \{y^n\})$ \hfill  via \eqref{eq:regularized_loss}
    \UNTIL{converged}
\end{algorithmic}
\end{algorithm}

The $i$-th symbol is assumed to be sampled from the resulting categorical distribution $\mathrm{Cat}(p_{i1}, \ldots, p_{ik})$.
We denote the resulting distribution over codes as $\tX$, and a code as $\tx$.
Instead of sampling $\tx \sim \tX$ during training, we use a loss function that optimizes the expected performance of the entire distribution $\tX$. 

We train the encoder by maximizing the \bemph{mutual information} between the distributions of codes $\tX$ and labels $Y$.
The mutual information is a symmetric quantity that measures the amount of information shared between two random variables.
It is defined as
\begin{align}
I(\tX; Y)
= H(\tX) - H(\tX|Y).
\label{eq:loss}
\end{align}
Since $\tX$ and $Y$ are discrete, their mutual information is bounded from both above and below as 
$0 \leq I(\tX; Y) \leq \log |\tX| = d \log k$.

To optimize the mutual information, the encoder directly computes empirical estimates of the two terms on the right-hand side of \eqref{eq:loss}.
Note that both terms consist of entropies of categorical distributions, which has the general closed-form formula 
\begin{align}
H(p_1, \ldots, p_n) = \sum_{i=1}^n p_i \log p_i.
\end{align}

Let $\overline{p}_{ij}$ be the empirical average of $p_{ij}$ calculated using data points in a batch.
Then, $\overline{p}_{ij}$ is an empirical estimate of the marginal distribution $p(\txval)$.
We compute the empirical estimate of $H(\tX)$ by adding its entropy estimate for each dimension.
\begin{align}
\widehat{H}(\tX)
&= \sum_{i=1}^d H(\overline{p}_{i1}, \ldots, \overline{p}_{ik}).
\label{eq:marginal_entropy}
\end{align}
We can also compute 
\begin{align}
\widehat{H}(\tX|Y) = \sum_{y=1}^c p(Y=y) \widehat{H}(\tX|Y=y),
\label{eq:conditional_entropy}
\end{align}
where $c$ is the number of classes.
The marginal probability $p(Y=y)$ is the frequency of class $y$ in the minibatch,
and $H(\tX | Y=y)$ can be computed by computing \eqref{eq:marginal_entropy} using only datapoints which belong to class $y$.
We emphasize that such a closed-form estimation of $I(\tX; Y)$ is only possible because we are using discrete codes. 
If $\tX$ was a continuous variable instead, we would have had to resort to approximtions of $I(\tX; Y)$ (e.g., \citet{belghazi2018mine}) to optimize it.

We briefly examine the loss function \eqref{eq:loss} to see why maximizing it results in discriminative $\tX$.
Maximizing $H(\tX)$ encourages the distribution of all codes to be as dispersed as possible,
while minimizing $H(\tX|Y)$ encourages the average embedding of each class to be as concentrated as possible.
Thus, the overall loss $I(\tX; Y)$ imposes a partitioning problem on the model:
it learns to split the entire probability space into regions with minimal overlap between different classes.
We will analyze our loss function more rigorously in \cref{subsec:mi_as_objective}. 

\subsection{Similarity Measure}
\label{subsec:similarity}

Suppose that all data points in the training set are assigned their codewords according to the rule (\ref{eq:argmax}).
Now we introduce how to compute a similarity between a query datapoint $\bx^{(q)}$ and a support datapoint $\bx^{(s)}$ so that it can be used for information retrieval or few-shot classification 
In the meta-learning literature, $\bx^{(q)}, \bx^{(s)}$ are also called test data and query data, respectively \citep{finn2017model}.

Denote by $\tx^{(s)}$ the codeword associated with $\bx^{(s)}$, constructed by (\ref{eq:argmax}).
For the test data  $\bx^{(q)}$, the encoder yields $p_{i,j}(\bx^{(q)})$ for $i=1,\ldots,d$ and $j=1,\ldots,k$.
As a similarity measure between $\bx^{(q)}$ and $\bx^{(s)}$, we calculate the following log probability
\begin{eqnarray}
\label{eq:eval}
\sum_{i=1}^d \log p_{i,\txval_i^{(s)}} (\bx^{(q)}).
\end{eqnarray}
The probabilistic quantity \eqref{eq:eval} indicates that  $\bx^{(q)}$ and $\bx^{(s)}$ become more similar 
when encoder's output, when $\bx^{(q)}$ is provided, is well aligned with $\tx^{(s)}$.

We can view our similarity measure \eqref{eq:eval} as a probabilistic generalization 
of the Hamming distance \citep{hamming1950error}. 
The Hamming distance quantifies the similarity between two strings of equal length as the number of positions at which the corresponding symbols are equal.
Because we have access to a distribution over codes, we use \eqref{eq:eval} to directly compute the log probability of having the same symbol at each position.

We use \eqref{eq:eval} as a similarity metric for both few-shot classification and image retrieval.
We perform few-shot classification by computing a codeword for each class via \eqref{eq:argmax} and classifying each test image by choosing the class that has the highest value of \eqref{eq:eval}.
We similarly perform image retrieval by mapping each support image to its most likely code \eqref{eq:argmax} and for each query image retrieving the support image that has the highest \eqref{eq:eval}. 

While we have described the operations in \eqref{eq:argmax} and \eqref{eq:eval} for a single pair $(\bx^{(q)}, \bx^{(s)})$,
one can easily parallelize our evaluation procedure since it is an argmax followed by a sum%
\footnote{We show a parallel implementation in the supplementary material.}.
Furthermore, $\tx$ typically requires little memory as it consists of discrete values, allowing us to compare against large support sets in parallel.
Experiments in \cref{subsec:retrieval} investigate the degree of DIMCO's efficiency in terms of both time and memory.

\subsection{Regularizing by Enforcing Independence}
\label{subsec:regularization}
One way of interpreting the code distribution $\tX$ is as a group of $d$ separate code distributions $\txval_1, \ldots, \txval_d$.
Note that the similarity measure described in \eqref{eq:eval} can be seen as ensemble of the similarity measures of these $d$ models.
A classic result in ensemble learning is that using more diverse learners increases ensemble performance \cite{kuncheva2003measures}.
In a similar spirit, we used an optional regularizer which promotes pairwise independence between each pair in these $d$ codes. 
Using this regularizer stabilized training, especially in more large-scale problems.

Specifically, we randomly sample pairs of indices $i_1, i_2$ from $\set{1, \ldots, d}$ during each forward pass.
Note that $\txval_{i_1} \otimes \txval_{i_2}$ and $(\txval_{i_1}, \txval_{i_2})$ are both categorical distributions with support size $k^2$, and that we can estimate the two different distributions within each batch.
We minimize their KL divergence to promote independence between these two distributions:
\be
\label{eq:indi_loss}
\calL_\textrm{ind}
= \kld
{\txval_{i_1} \otimes \txval_{i_2}}
{\txval_{i_1}, \txval_{i_2}}.
\ee
We compute \eqref{eq:indi_loss} for a fixed number of random pairs of indices for each batch.
The cost of computing this regularization term is miniscule compared to that of other components such as feeding data through the encoder.

Using this regularizer in conjunction with the learning objective \eqref{eq:loss} yields the following regularized loss:
\begin{align}
\label{eq:regularized_loss}
\calL 
&= -I(\tX; Y) + \lambda \calL_\textrm{ind} \nonumber \\
&= -H(\tX) + H(\tX|Y) + \lambda \calL_\textrm{ind}.
\end{align}
We fix $\lambda=1$ in all experiments; we found that DIMCO's performance was not particularly sensitive to this hyperparameter.
We emphasize that while this optional regularizer stabilizes training, our learning objective is the mutual information $I(\tX; Y)$ in \eqref{eq:loss}.

\subsection{Visualization of Codes}
\begin{figure}[t]
\centering
    \includegraphics[width=1.0\linewidth]{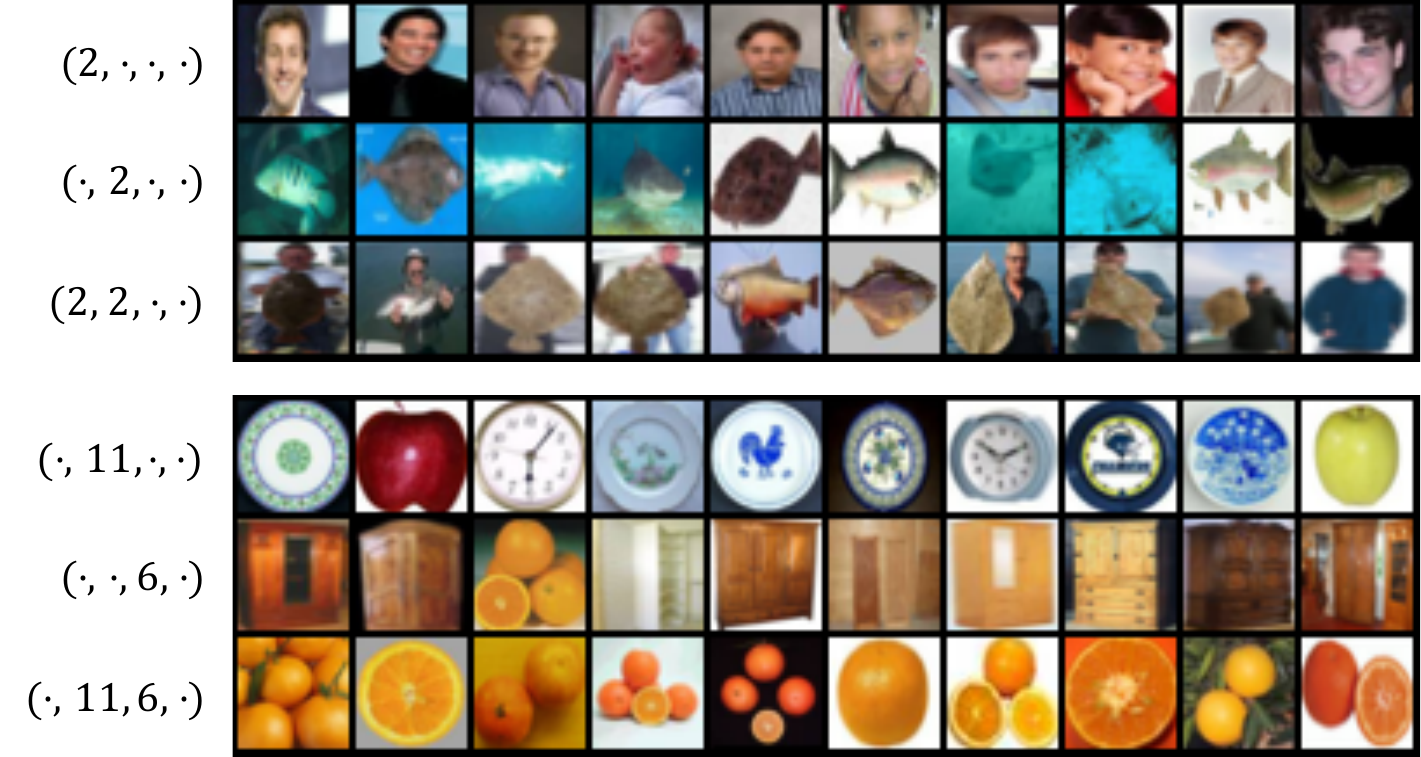}
\caption{
    Visualization of codes of a DIMCO model ($k=16$, $d=4$) trained on CIFAR100.
    Each row shows the top 10 images that assign highest marginal probability to specific codes (shown on left).
    We also visualize all $kd$ codewords in the supplementary material.
}
\label{fig:codevis}
\end{figure}
\label{subsec:codevis}
In \cref{fig:codevis}, we show images retrieved using our similarity measure \eqref{eq:eval}.
We trained a DIMCO model ($k=16$, $d=4$) on the CIFAR100 dataset.
We select specific code locations and plot the top $10$ test images according to our similarity measure.
For example, the top-1 (leftmost) image for code $(\cdot, j_2, \cdot, j_4)$ would be computed as
\begin{align}
    \argmax_{n \in \{1, \ldots, N\}} \left[ \log p_{2,j_2} (\bx^n) + \log p_{4, j_4} (\bx^n) \right],
\end{align}
where $N$ is the number of test images.

We visualize two different combinations of codes in \cref{fig:codevis}.
The two examples show that using codewords together results in their respective semantic concepts to be combined:
(man + fish = man holding fish),
(round + warm color = orange).
While we visualized combinations of $2$ codewords for clarity, DIMCO itself uses a combination of $d$ such codewords. 
The regularizer described in \cref{subsec:regularization} further encourages 
each of these $d$ codewords to represent different concepts.
The combinatorially many ($k^d$) combinations in which DIMCO can assemble such codewords gives DIMCO sufficient expressive power to solve challenging tasks.

\section{Analysis} \label{sec:analysis}

\subsection{Is Mutual Information a Good Objective?}
\label{subsec:mi_as_objective}

Our learning objective for DIMCO \eqref{eq:loss} is the mutual information between codes and labels.
In this subsection, we justify this choice by showing that many previous objectives are closely related to mutual information.
Due to space constraints, we only show high-level connections here and provide a more detailed exposition in the supplementary material.

\paragraph{Cross-entropy}
The de facto loss for classification is the cross-entropy loss, which is defined as
\begin{align}
\mathrm{xent}(Y, X) 
= -\E_{y \sim Y, x \sim X} \left[ \log q(\widehat{Y} = y | \tX(x)) \right],
\end{align}
where $\widehat{Y}$ is the model's prediction of $Y$.
Using the observation that the final layer $q(\cdot)$ acts a parameterized approximation to the true conditional distribution 
$p(Y|\tX)$, we write this as
\be
\mathrm{xent}(Y, X)
&\approx &\E_{y \sim Y, x \sim X} \left[ -\log p(\y|\tx) \right] \nonumber \\
&=& - I(\tX; Y) + H(Y).
\ee
The $H(Y)$ term can be ignored since it is not affected by model parameters.
Therefore, minimizing cross-entropy is approximately equivalent to maximizing mutual information.
The two objectives become completely equivalent when the final linear layer $q(\cdot)$ perfectly represents the conditional distribution $q(\y|\tx)$.
Note that for discrete $\tx$, we cannot use a linear layer to parameterize $q(\y|\tx)$, and therefore cannot directly optimize the cross-entropy loss.
We can therefore view our loss as a necessary modification of the cross-entropy loss for our setup of using discrete embeddings.

\paragraph{Contrastive Losses}
Many metric learning methods \citep{koch2015siamese,hoffer2015deep,sohn2016improved,movshovitz2017no,duan2018deep} use a contrastive learning objective to learn a continuous embedding ($\tX$).
Such contrastive losses consist of 
(1) a positive term that encourages an embedding to move closer to that of other relevant embeddings and 
(2) a negative term that encourages it to move away from irrelevant embeddings.
The positive term approximately minimizes $\log p(\tX|y)$ while the negative term as approximately minimizes $\log p(\tX)$.
Together, these terms have the combined effect of maximizing
\be
\lefteqn{ \E \left[ \log p(\tx|y) - \log p(\tx) \right] } \nonumber \\
& = -H(\tX|Y) + H(\tX) = I(\tX; Y).
\ee
We show such equivalences in detail in the supplementary material.

In addition to these direct connections to previous loss functions, we show empirically in \cref{subsec:correlation} that the mutual information strongly correlates with both the top-1 accuracy metric for classification and the Recall@1 metric for retrieval.

\subsection{Does Using Discrete Codes Help Generalization?}
\label{subsec:generalization}

In \cref{sec:intro}, we have motivated the use of discrete codes through the regularization effect of an information bottleneck.
In this subsection, we prove a PAC bound to theoretically analyze whether learning discrete codes by maximizing mutual information leads to better generalization.
In particular, we study how the mutual information on the test set is affected by the choice of input dataset structure and code hyperparameters $k, d$.

We analyze DIMCO's characteristics by examining each minibatch.
Following related meta-learning works \citep{amit2017meta,ravi2018amortized}, we call each batch a ``task''.
We emphasize that this is only a difference in naming convention.
The analysis in this subsection applies equally well to the metric learning setup; 
we can view each batch consisting of support and query points as a task.

Define a task $T$ to be a distribution over $\calZ = \calX \times \calY$.
Let tasks $T^1, \ldots, T^n$ be sampled i.i.d. from a distribution of tasks $\tau$.
Each task $T$ consists of a fixed-size dataset $D_T = z_T^1, \ldots, z_T^m = (x_T^1, y_T^1), \ldots, (x_T^m, y_T^m)$, 
which is a set of $m$ i.i.d. samples from the data distribution ($z_T^j \sim T$).
Let $\theta$ be the parameters of DIMCO.
Let $X, Y, \tX$ be the random variables for data, labels, and codes, respectively.
Recall that our objective is the expected mutual information between labels and codes:
\begin{align}
\calL(\tau, \theta) 
= -\E_{T \sim \tau} \left[ I(\tX(X_T, \theta); Y_T) \right].
\label{eq:mi_objective}
\end{align}
The loss that we actually optimize (\cref{eq:marginal_entropy,eq:conditional_entropy}) is the empirical loss
\begin{align}
\widehat{\calL}(T^{1:n}, \theta) 
= - \frac{1}{n} \sum_{i=1}^n \hat{I}(\tX(X_{T^i}, \theta); Y_{T^i}).
\end{align}
The following theorem bounds the difference between these the expected loss $\calL$ and the empirical loss $\hat{\calL}$.

\begin{theorem}[simplified]
\label{thm:gen_bound}
Let $d_\Theta$ be the VC dimension of the encoder $\tX(\cdot)$.
The following inequality holds with high probability:
\begin{eqnarray}
\calL(\tau, \theta) - \widehat{\calL} (T^{1:n}, \theta) 
\leq
O \left( \sqrt{\frac{d_\Theta}{n} \log \frac{n}{d_\Theta}} \right) + \nonumber \\
O\left( \frac{|\tX| \log(m)}{\sqrt{m}} \right) + O\left( \frac{|\tX| |Y|}{m} \right).
\label{eq:bound}
\end{eqnarray}
\end{theorem}
\begin{proof}
We use VC dimension bounds and a finite sample bound for mutual information \citep{shamir2010learning}.
We defer a detailed statement and proof to the supplementary material.
\end{proof}

First note that all three terms in our generalization gap \eqref{eq:bound} converge to zero as $n,m \rightarrow \infty$.
This shows that training a model by maximizing empirical mutual information as in \cref{eq:marginal_entropy,eq:conditional_entropy} generalizes perfectly in the limit of infinite data.

\Cref{thm:gen_bound} also shows how the generalization gap is affected differently by dataset size $m$ and number of datasets $n$.
A large $n$ directly compensates for using a large backbone ($d_\Theta$), while a large $m$ compensates for using a large final representation ($\tX$).
Put differently, to effectively learn from small datasets ($m$), one should use a small representation $(\tX)$.
The number of datasets $n$ is typically less of a problem because the number of different ways to sample datasets is combinatorially large (e.g., $n>10^{10}$ for miniImagenet $5$-way $1$-shot tasks).
Recall that DIMCO has $|\tX| = d \log k$, meaning that we can control the latter two terms using our hyperparameters $d, k$.
We have motivated the use of discrete codes through the information bottleneck effect of small codes $\tX$, and \cref{thm:gen_bound} confirms this intuition.

\section{Related Work} \label{sec:previous}
\paragraph{Information Bottleneck}
\nocite{alemi2016deep,goyal2019infobot}
DIMCO and \cref{thm:gen_bound} are both close in spirit to the information bottleneck (IB) principle \citep{tishby2000information,tishby2015deep,shwartz2017opening}.
IB finds a set of compact representatives $\tX$ while maintaining sufficient information about $Y$,
minimizing the following objective function
\begin{align}
\label{eq:ibf}
\calJ( p(\tx | \x) ) 
= I(\tX; X) - \beta I(\tX; Y),
\end{align}
subject to $\sum_{\tx} p(\tx|\x) = 1$.
Equivalently, it can be stated that one maximizes $I(\tX; Y)$ while simultaneously minimizing $I(\tX; X)$.
Similarly, our objective \eqref{eq:mi_objective} is information maximization $I(\tX; Y)$,
while our bound \eqref{eq:bound} suggests that the representation capacity $|\tX|$ should be low for generalization.
In the deterministic information bottleneck \citep{strouse2017deterministic}, $I(\tX; X)$ is replaced by $H(\tX)$.
These three approaches to generalization are related via the chain of inequalities $I(\tX; X) \leq H(\tX) \leq \log |\tX|$,
which is tight in the limit of $\tX$ being imcompressible.
For any finite representation, i.e., $|\tX|=N$, the limit $\beta \rightarrow \infty$ in \eqref{eq:ibf} yields a hard partitioning of $X$ into $N$ disjoint sets.
DIMCO uses the infomax principle to learn $N=k^d$ such representatives, which are arranged by $k$-way $d$-dimensional discrete codes for compact representation with sufficient information on $Y$.

\paragraph{Meta-Regularization}
\nocite{finn2017model,finn2017meta,amit2017meta,triantafillou2019meta}
Previous meta-learning methods have restricted task-specific learning by
learning only a subset of the network \citep{Lee2018},
learning on a low-dimensional latent space \citep{rusu2018meta},
learning on a meta-learned prior distribution of parameters \citep{kim2018bayesian},
and learning context vectors instead of model parameters \citep{zintgraf2018caml}.
Our analysis in \cref{thm:gen_bound} suggests that reducing the expressive power of the task-specific learner has a meta-regularizing effect,
indirectly giving theoretic support for previous works that benefitted from reducing the expressive power of task-specific learners.

\begin{figure}
\centering
\includegraphics[width=.85\linewidth]{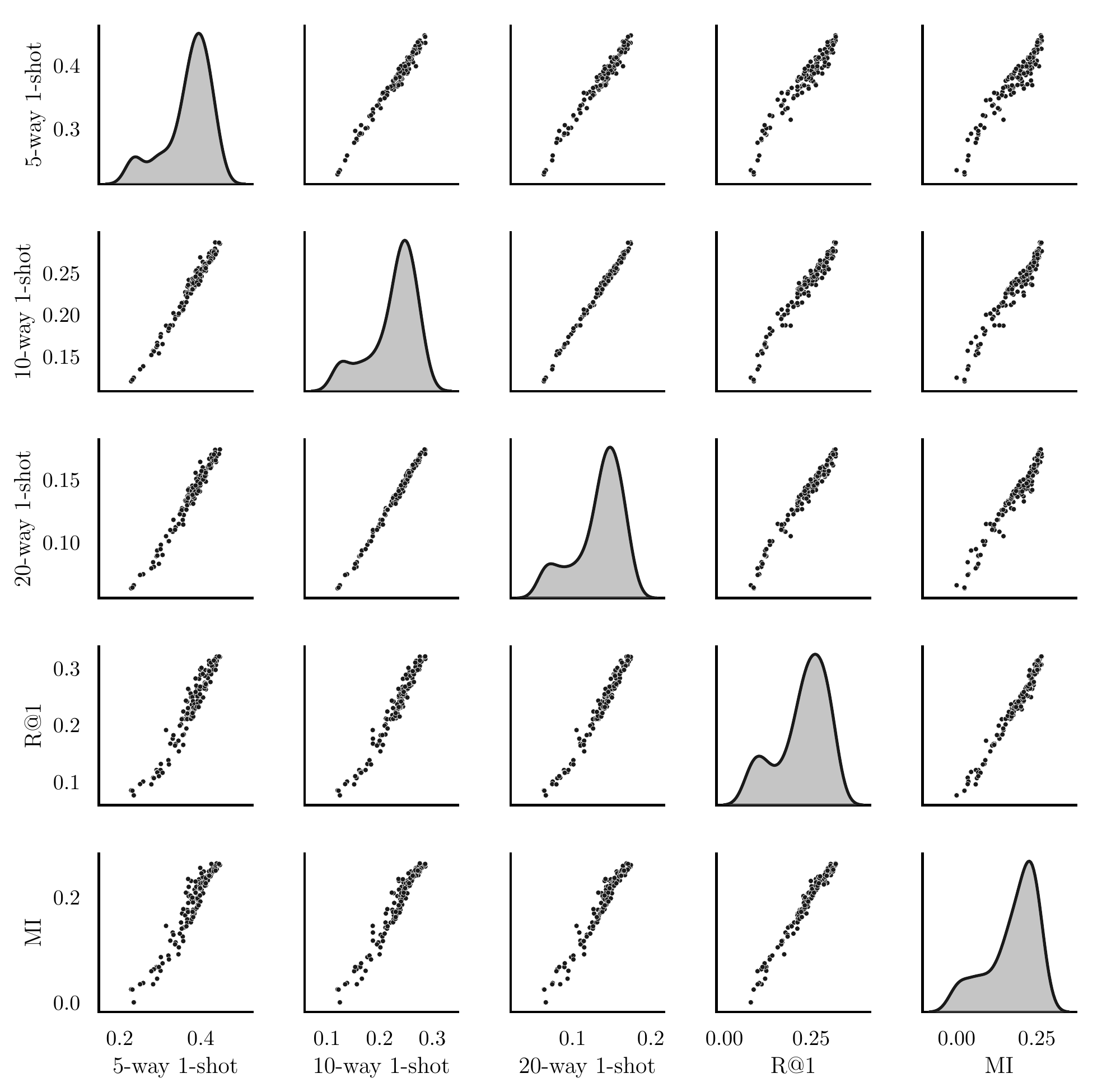}
\vspace{-5pt}
\caption{
    Pairwise correlation between $\textrm{MI}=I(X; \tX)$ and previous metrics.
    Best viewed zoomed in.
}
\vspace{-10pt}
\label{fig:corr}
\end{figure}

\paragraph{Discrete Representations}
Discrete representations have been thoroughly studied in information theory \citep{shannon1948mathematical}.
Recent deep learning methods directly learn discrete representations by learning generative models with discrete latent variables \citep{rolfe2016discrete,van2017neural,razavi2019generating}
or maximizing the mutual information between representation and data \citep{hu2017learning}.
DIMCO is related to but differs from these works as it assumes a supervised meta-learning setting and performs infomax using \textit{labels} instead of data.

A standard approach to learning label-aware discrete codes is to first learn continuous embeddings and then quantize it using an objective that maximally preserves its information \citep{gray1998quantization,jegou2011product,gong2012iterative}.
DIMCO can be seen as an end-to-end alternative to quantization which directly learns discrete codes.
\citet{jeong2018efficient} similarly learns a sparse binary code in an end-to-end fashion by solving a minimum cost flow problem with respect to labels.
Their method differs from DIMCO, which learns a dense discrete code by optimizing $I(\tX; Y)$, which we estimate with a closed-form formula.

\paragraph{Metric Learning}
The structure and loss function of DIMCO is closely related to that of metric learning methods \citep{hoffer2015deep,sohn2016improved,vinyals2016matching,duan2018deep}.
We show that the loss functions of these methods can be seen as approximation to the mutual information ($I(\tX; Y)$) in \cref{subsec:loss}, and provide a more in-depth exposition in the supplementary material.
While all of these previous methods require a support/query split within each batch,
DIMCO simply optimizes an information-theoretic quantity of each batch, removing the need for such structured batch construction.

\paragraph{Information Theory and Representation Learning}
Many works have applied information-theoretic principles to unsupervised representation learning:
to derive an objective for GANs to learn disentangled features \citep{chen2016infogan},
to analyze the evidence lower bound (ELBO) \citep{alemi2017fixing,chen2018isolating},
and to directly learn representations \citep{alemi2016deep,hjelm2018learning,oord2018representation,grover2018uncertainty,choi2019neural}.
DIMCO is also an information-theoretic representation learning method, but we instead assume a supervised learning setup where the representation must reflect ground-truth labels.
We also used previous results from information theory to prove a generalization bound for our representation learning method.

\section{Experiments} \label{sec:experiments}
\begin{table}[ht!]
\vspace{-7pt}
\caption{
    Top-1 accuracies under various compressed CIFAR-10/100 embeddings with size ($k, d$).
    The original embeddings were continuous $64$-dimensional embeddings ($k \approx 2^{32}, d=64$).
}
\vspace{4pt}
\label{tab:comp_cifar_small}
\centering \begin{tabular}{
    ll
    cccc
    }
\toprule
    & 
    & \multicolumn{2}{c}{CIFAR-10} & \multicolumn{2}{c}{CIFAR-100} \\
      \cmidrule(l{4pt}r{4pt}){3-4}   \cmidrule(l{4pt}r{4pt}){5-6}

    & & $d=4$ & $d=16$ & $d=4$ & $d=16$ \\
\midrule PQ
    & {$k=2$} & 50.38 & 87.24 & 7.77 & 32.47 \\ 
    & {$k=4$} & 86.86 & 90.92 & 18.78 & 52.02 \\ 
    & {$k=8$} & 90.44 & 91.49 & 31.04 & 58.82 \\ 
    & {$k=16$}& 91.02 & 91.52 & 43.16 & 61.00 \\
\midrule DIMCO 
    & {$k=2$} & 64.43 & 88.85 & 10.03 & 33.46 \\ (ours) 
    & {$k=4$} & 88.47 & 91.04 & 25.7  & 53.36 \\ 
    & {$k=8$} & 91.29 & 91.45 & 46.22 & 58.84 \\ 
    & {$k=16$} & \BF 91.68 & \BF 91.57 & \BF 57.83 & \BF 62.49 \\
\bottomrule
\end{tabular}
\caption{
    Top-$1$ accuracies and compression rates of ImageNet embeddings under various compressed embedding sizes ($k, d$).
    The compression rate is the ratio between uncompressed and compressed sizes; it is calculated as $\frac{32 \cdot 2048}{d \log k}$.
}
\vspace{4pt}
\centering \begin{tabular}{
    ll
    S[table-format=5.0, detect-weight=true]
    S[table-format=2.2, round-precision=2, detect-weight=true]
    }
\toprule
Method  & ($k, d$)    & {Compression Rate} 
                                & {Accuracy}  \\
\midrule
SQ      & ($2, 2048$) & 32      & 0.10        \\
SQ      & ($4, 2048$) & 16      & 0.28        \\
SQ      & ($8, 2048$) & 10      & 12.36       \\
SQ      & ($16, 2048$)& 8       & 57.80       \\
PQ      & ($2, 2$)    & 32768   & 0.21        \\
PQ      & ($4, 4$)    & 8192    & 0.92        \\
PQ      & ($8, 8$)    & 2730    & 11.91       \\
PQ      & ($16, 16$)  & 1024    & 44.93       \\
\midrule
DIMCO   & ($2, 2$)    & 32768   & 0.14        \\
DIMCO   & ($4, 4$)    & 8192    & 2.68        \\
DIMCO   & ($8, 8$)    & 2730    & 33.28       \\
DIMCO   & ($16, 16$)  & 1024    & \BF 63.64   \\
\bottomrule
\end{tabular}
\vspace{-5pt}
\label{tab:comp_imagenet}
\end{table}

\begin{table}[t]
\vspace{-7pt}
\caption{
    Few-shot classification accuracies on the miniImageNet benchmark.
    $\dagger$ denotes transductive methods, which are more expressive by taking unlabeled examples into account.
}
\label{tab:fewshot}
\vspace{3pt}
\centering \begin{tabular}{lcc}
\toprule                    %
    Method                  & $5$-way $1$-shot      & $5$-way $5$-shot      \\
\midrule %
    $^\dagger$ TPN \nocite{liu2018transductive}
                            & 55.51 $\pm$ 0.86      & 69.86 $\pm$ 0.65      \\
    $^\dagger$ FEAT \nocite{ye2018learning}
                            & 55.75 $\pm$ 0.20      & 72.17 $\pm$ 0.16      \\
\midrule
    MetaLSTM \nocite{ravi2016optimization}             
                            & 43.44 $\pm$ 0.77      & 60.60 $\pm$ 0.71      \\
    MatchingNet \nocite{vinyals2016matching}             
                            & 43.56 $\pm$ 0.84      & 55.31 $\pm$ 0.73      \\
    ProtoNet \nocite{snell2017prototypical}               
                            & 49.42 $\pm$ 0.78      & 68.20 $\pm$ 0.66      \\
    RelationNet \nocite{sung2018learning}               
                            & 50.44 $\pm$ 0.82      & 65.32 $\pm$ 0.70      \\
    R2D2 \nocite{bertinetto2018meta}               
                            & 51.2 $\pm$ 0.6      & \BF 68.8 $\pm$ 0.1      \\
    MetaOptNet-SVM \nocite{lee2019meta}          
                            & 52.87 $\pm$ 0.57      & 68.76 $\pm$ 0.48  \\
\midrule
    DIMCO ($64, 64$)        & 47.33 $\pm$ 0.46      & 61.59 $\pm$ 0.52      \\
    DIMCO ($64, 128$)       & \BF 53.29 $\pm$ 0.47  & 64.79 $\pm$ 0.57      \\
\bottomrule
\end{tabular}
\end{table}

\begin{table*}[t]
\vspace{-7pt}
\caption{
    Image retrieval performance on CUB-200-2011 and Cars-196, measured by Recall@1.
    Memory is the number of bits that an embedding vector of each image uses. 
    Time is seconds taken to retrieve a single query from database 
    (5,924 and 8,131 images for CUB-200-2011 and Cars-196, respectively).
}
\label{tab:metric_learning}
\vspace{5pt}
\centering \begin{tabular}{llc cc cc}
\toprule
    & & & \multicolumn{2}{c}{CUB-200-2011} & \multicolumn{2}{c}{Cars-196} \\ 
    \cmidrule(l{4pt}r{4pt}){4-5} \cmidrule(l{4pt}r{4pt}){6-7}  
    Method & ($k,d$) & Memory [bits]
                            & Recall@1 & Time [s]
                                            & Recall@1 & Time [s] \\
\midrule
    Binomial Deviance %
    & (-, 128)      & 4096  & 57.25 & 16.37 & 72.53 & 21.86 \\
    Triplet %
    & (-, 128)      & 4096  & 56.80 & 16.37 & 73.79 & 21.86 \\
    Proxy-NCA %
    & (-, 128)      & 4096  & 56.19 & 16.37 & 75.94 & 21.86 \\ 
\midrule
    \multirow{4}{*}{DIMCO (ours)} 
    & (32, 32)      & 160   & 51.04 & 1.48 & 63.44 & 2.64   \\ 
    & (64, 64)      & 384   & 55.78 & 2.85 & 72.06 & 6.01   \\ 
    & (128, 128)    & 896   & 58.05 & 5.81 & 76.04 & 11.92 \\ 
    & (256, 256)    & 2048  & \BF 58.90 & 12.20 & \BF 77.32 & 16.04 \\
\bottomrule
\vspace{-20pt}
\end{tabular}
\end{table*}

In our experiments, we use datasets with varying degrees of complexity:
CIFAR10/100~\citep{Krizhevsky09learningmultiple}, 
miniImageNet~\citep{vinyals2016matching}, 
CUB200~\citep{wah2011caltech}, 
Cars196~\citep{KrauseStarkDengFei-Fei_3DRR2013}, 
and ImageNet (ILSVRC-2012-CLS, \citet{imagenet_cvpr09}).
We use standard train/test splits for each dataset unless stated otherwise.
We also use various network architectures: 
4-layer convnet~\citep{vinyals2016matching} 
and ResNet12/20/50~\citep{he2016deep,mishra2017simple}.
We followed previously reported experimental setups as closely as possible,
and provide minor experiment details in the supplementary material.

\subsection{Correlation of Metrics} \label{subsec:correlation}
We have shown in \cref{subsec:mi_as_objective} that the mutual information $I(\tX; Y)$ is strongly connected to previous loss functions for classification and retrieval.
In this subsection, we perform experiments to verify whether $I(\tX; Y)$ is a good \textit{metric} that quantitatively shows the quality of the representation $\tX$.
We trained DIMCO on the miniImageNet dataset with $k = d = 64$ for $20$ epochs.
We plot the pairwise correlations between five different metrics:
($5, 10, 20$)-way $1$-shot accuracy, $\textrm{Recall@}1$, and $I(\tX; Y)$.
The results in \cref{fig:corr} show that all five metrics are very strongly correlated.
We observed similar trends when training with loss functions other than $I(\tX; Y)$ as well; 
we show these experiments in the appendix due to space constraints.

\subsection{Label-Aware Bit Compression}
We applied DIMCO to compressing feature vectors of trained classifier networks.
We obtained penultimate embeddings of ResNet20 networks each trained on CIFAR10 and CIFAR100.
The two networks had top-1 accuracies of $91.65$ and $66.61$, respectively.
We trained on embeddings for the train set of each dataset, and measured top-1 accuracy of the test set using the training set as support.
We compare DIMCO to product quantization (PQ, \citet{jegou2011product}), which similarly compresses a given embededing to a $k$-way $d$-dimensional code.
We compare the two methods in \cref{tab:comp_cifar_small} with the same range of $k, d$ hyperparameters.
We performed the same experiment on the larger ImageNet dataset with a ResNet50 network which had a top-1 accuracy of $76.00$.
We compare DIMCO to both adaptive scalar quantization (SQ) and PQ in \cref{tab:comp_imagenet}.
We show extended experiments for all three datasets in the supplementary material.

The results in \cref{tab:comp_cifar_small} and \cref{tab:comp_imagenet} demonstrate that DIMCO consistently outperforms PQ, and is especially efficient when $d$ is low.
Furthermore, the ImageNet experiment (\cref{tab:comp_imagenet}) shows that DIMCO even outperforms SQ, which has a much lower compression rate compared to the embedding sizes we consider for DIMCO.
These results are likely due to DIMCO performing \textit{label-aware compression} where it compresses the embedding while taking the label into account, whereas PQ and SQ only compress the embeddings themselves.

\subsection{Few-shot Classification} \label{subsec:fewshot}
We evaluated DIMCO's few-shot classification performance on the miniImageNet dataset. 
We compare against the following previous works: 
\citet{
    ravi2016optimization,
    vinyals2016matching,
    snell2017prototypical,
    sung2018learning,
    bertinetto2018meta,
    liu2018transductive,
    ye2018learning,
    lee2019meta}.
All methods use the standard four-layer convnet with $64$ filters per layer\footnote{
    some methods used more filters; we used $64$ for fairness.
}.
We use the data augmentation scheme proposed by \citet{lee2019meta} and use balanced batches of $100$ images consisting of $10$ different classes.
We evaluate on both $5$-way $1$-shot and $5$-way $5$-shot learning, and report $95\%$ confidence intervals of $1000$ random episodes on the test split.

Results are shown in \cref{tab:fewshot}, and we provide an extended table with an alternative backbone in the supplementary material.
\Cref{tab:fewshot} shows that DIMCO outperforms previous works on the $5$-way $1$-shot benchmark.
DIMCO's $5$-way $5$-shot performance is relatively low, likely because the similarity metric (\cref{subsec:similarity}) handles support datapoints individually instead of aggregating them, similarly to Matching Nets \citep{vinyals2016matching}.
Additionally, other methods are explicitly trained to optimize $5$-shot performance, whereas DIMCO's training procedure is the same regardless of task structure.

\subsection{Image Retrieval} \label{subsec:retrieval}
We conducted image retrieval experiments using two standard benchmark datasets: CUB-200-2011 and Cars-196.
As baselines, we consider three widely adopted metric learning methods:
Binomial Deviance \citep{yideep2014},
Triplet loss \citep{hoffer2015deep},
and Proxy-NCA \citep{movshovitz2017no}.
The backbone for all methods was a ResNet-50 network pretrained on the ImageNet dataset.
We trained DIMCO on various combinations of $(p, d)$, and set the embedding dimension of the baseline methods to 128.
We measured the time per query for each method on a Xeon E5-2650 CPU without any parallelization.
We note that computing the retrieval time using a parallel implementation would skew the results even more in favor of DIMCO since DIMCO's evaluation is simply one memory access followed by a sum.

Results presented in \cref{tab:metric_learning} show that DIMCO outperforms all three baseline,
and that the compact code of DIMCO takes roughly an order of magnitude less memory, and requires less query time as well.
This experiment also demonstrates that discrete representations can outperform modern methods that use continuous embeddings, even on this relatively large-scale task.
Additionally, this experiment shows that DIMCO can train using large backbones without significantly overfitting.

\section{Conclusion} \label{sec:discussion}
We introduced DIMCO, a model that learns a discrete representation of data by directly optimizing the mutual information with the label.
To evaluate our initial intuition that shorter representations generalize better between tasks, we provided generalization bounds that get tighter as the representation gets shorter.
Our thorough experiments demonstrated that DIMCO is effective at both compressing a continuous embedding, and also at learning a discrete embedding from scratch in an end-to-end manner.
The discrete embeddings of DIMCO outperformed recent continuous methods while also being more efficient in terms of both memory and time.
We believe the tradeoff between discrete and continuous embeddings is an exciting area for future research.

DIMCO was motivated by concepts such as the minimum description length (MDL) principle and the information bottleneck:
compact \textit{task representations} should have less room to overfit.
Interestingly, \citet{yin2019meta} reports that doing the opposite%
\textemdash%
regularizing the \textit{task-general parameters}%
\textemdash
prevents meta-overfitting by discouraging the meta-learning model from memorizing the given set of tasks.
In future work, we will investigate the common principle underlying these seemingly contradictory approaches for a fuller understanding of meta-generalization.

\bibliographystyle{preamble/icml2020} 
\bibliography{main}

\begin{thebibliography}{59}
\providecommand{\natexlab}[1]{#1}
\providecommand{\url}[1]{\texttt{#1}}
\expandafter\ifx\csname urlstyle\endcsname\relax
  \providecommand{\doi}[1]{doi: #1}\else
  \providecommand{\doi}{doi: \begingroup \urlstyle{rm}\Url}\fi

\bibitem[Alemi et~al.(2016)Alemi, Fischer, Dillon, and Murphy]{alemi2016deep}
Alemi, A.~A., Fischer, I., Dillon, J.~V., and Murphy, K.
\newblock Deep variational information bottleneck.
\newblock \emph{arXiv preprint arXiv:1612.00410}, 2016.

\bibitem[Alemi et~al.(2017)Alemi, Poole, Fischer, Dillon, Saurous, and
  Murphy]{alemi2017fixing}
Alemi, A.~A., Poole, B., Fischer, I., Dillon, J.~V., Saurous, R.~A., and
  Murphy, K.
\newblock Fixing a broken elbo.
\newblock \emph{arXiv preprint arXiv:1711.00464}, 2017.

\bibitem[Amit \& Meir(2017)Amit and Meir]{amit2017meta}
Amit, R. and Meir, R.
\newblock Meta-learning by adjusting priors based on extended pac-bayes theory.
\newblock \emph{arXiv preprint arXiv:1711.01244}, 2017.

\bibitem[Belghazi et~al.(2018)Belghazi, Baratin, Rajeswar, Ozair, Bengio,
  Courville, and Hjelm]{belghazi2018mine}
Belghazi, M.~I., Baratin, A., Rajeswar, S., Ozair, S., Bengio, Y., Courville,
  A., and Hjelm, R.~D.
\newblock Mine: mutual information neural estimation.
\newblock \emph{arXiv preprint arXiv:1801.04062}, 2018.

\bibitem[Bertinetto et~al.(2018)Bertinetto, Henriques, Torr, and
  Vedaldi]{bertinetto2018meta}
Bertinetto, L., Henriques, J.~F., Torr, P.~H., and Vedaldi, A.
\newblock Meta-learning with differentiable closed-form solvers.
\newblock \emph{arXiv preprint arXiv:1805.08136}, 2018.

\bibitem[Chen et~al.(2018)Chen, Li, Grosse, and Duvenaud]{chen2018isolating}
Chen, T.~Q., Li, X., Grosse, R.~B., and Duvenaud, D.~K.
\newblock Isolating sources of disentanglement in variational autoencoders.
\newblock In \emph{Advances in Neural Information Processing Systems}, pp.\
  2610--2620, 2018.

\bibitem[Chen et~al.(2016)Chen, Duan, Houthooft, Schulman, Sutskever, and
  Abbeel]{chen2016infogan}
Chen, X., Duan, Y., Houthooft, R., Schulman, J., Sutskever, I., and Abbeel, P.
\newblock Infogan: Interpretable representation learning by information
  maximizing generative adversarial nets.
\newblock In \emph{Advances in neural information processing systems}, pp.\
  2172--2180, 2016.

\bibitem[Choi et~al.(2019)Choi, Tatwawadi, Grover, Weissman, and
  Ermon]{choi2019neural}
Choi, K., Tatwawadi, K., Grover, A., Weissman, T., and Ermon, S.
\newblock Neural joint source-channel coding.
\newblock In \emph{International Conference on Machine Learning}, pp.\
  1182--1192, 2019.

\bibitem[Deng et~al.(2009)Deng, Dong, Socher, Li, Li, and
  Fei-Fei]{imagenet_cvpr09}
Deng, J., Dong, W., Socher, R., Li, L.-J., Li, K., and Fei-Fei, L.
\newblock {ImageNet: A Large-Scale Hierarchical Image Database}.
\newblock In \emph{CVPR09}, 2009.

\bibitem[Duan et~al.(2018)Duan, Zheng, Lin, Lu, and Zhou]{duan2018deep}
Duan, Y., Zheng, W., Lin, X., Lu, J., and Zhou, J.
\newblock Deep adversarial metric learning.
\newblock In \emph{Proceedings of the IEEE Conference on Computer Vision and
  Pattern Recognition}, pp.\  2780--2789, 2018.

\bibitem[Finn \& Levine(2017)Finn and Levine]{finn2017meta}
Finn, C. and Levine, S.
\newblock Meta-learning and universality: Deep representations and gradient
  descent can approximate any learning algorithm.
\newblock \emph{arXiv preprint arXiv:1710.11622}, 2017.

\bibitem[Finn et~al.(2017)Finn, Abbeel, and Levine]{finn2017model}
Finn, C., Abbeel, P., and Levine, S.
\newblock Model-agnostic meta-learning for fast adaptation of deep networks.
\newblock \emph{arXiv preprint arXiv:1703.03400}, 2017.

\bibitem[Gong et~al.(2012)Gong, Lazebnik, Gordo, and
  Perronnin]{gong2012iterative}
Gong, Y., Lazebnik, S., Gordo, A., and Perronnin, F.
\newblock Iterative quantization: A procrustean approach to learning binary
  codes for large-scale image retrieval.
\newblock \emph{IEEE Transactions on Pattern Analysis and Machine
  Intelligence}, 35\penalty0 (12):\penalty0 2916--2929, 2012.

\bibitem[Goyal et~al.(2019)Goyal, Islam, Strouse, Ahmed, Botvinick, Larochelle,
  Levine, and Bengio]{goyal2019infobot}
Goyal, A., Islam, R., Strouse, D., Ahmed, Z., Botvinick, M., Larochelle, H.,
  Levine, S., and Bengio, Y.
\newblock Infobot: Transfer and exploration via the information bottleneck.
\newblock \emph{arXiv preprint arXiv:1901.10902}, 2019.

\bibitem[Gray \& Neuhoff(1998)Gray and Neuhoff]{gray1998quantization}
Gray, R.~M. and Neuhoff, D.~L.
\newblock Quantization.
\newblock \emph{IEEE transactions on information theory}, 44\penalty0
  (6):\penalty0 2325--2383, 1998.

\bibitem[Grover \& Ermon(2018)Grover and Ermon]{grover2018uncertainty}
Grover, A. and Ermon, S.
\newblock Uncertainty autoencoders: Learning compressed representations via
  variational information maximization.
\newblock \emph{arXiv preprint arXiv:1812.10539}, 2018.

\bibitem[Hamming(1950)]{hamming1950error}
Hamming, R.~W.
\newblock Error detecting and error correcting codes.
\newblock \emph{The Bell system technical journal}, 29\penalty0 (2):\penalty0
  147--160, 1950.

\bibitem[He et~al.(2016)He, Zhang, Ren, and Sun]{he2016deep}
He, K., Zhang, X., Ren, S., and Sun, J.
\newblock Deep residual learning for image recognition.
\newblock In \emph{Proceedings of the IEEE conference on computer vision and
  pattern recognition}, pp.\  770--778, 2016.

\bibitem[Hjelm et~al.(2018)Hjelm, Fedorov, Lavoie-Marchildon, Grewal,
  Trischler, and Bengio]{hjelm2018learning}
Hjelm, R.~D., Fedorov, A., Lavoie-Marchildon, S., Grewal, K., Trischler, A.,
  and Bengio, Y.
\newblock Learning deep representations by mutual information estimation and
  maximization.
\newblock \emph{arXiv preprint arXiv:1808.06670}, 2018.

\bibitem[Hoffer \& Ailon(2015)Hoffer and Ailon]{hoffer2015deep}
Hoffer, E. and Ailon, N.
\newblock Deep metric learning using triplet network.
\newblock In \emph{International Workshop on Similarity-Based Pattern
  Recognition}, pp.\  84--92. Springer, 2015.

\bibitem[Hu et~al.(2017)Hu, Miyato, Tokui, Matsumoto, and
  Sugiyama]{hu2017learning}
Hu, W., Miyato, T., Tokui, S., Matsumoto, E., and Sugiyama, M.
\newblock Learning discrete representations via information maximizing
  self-augmented training.
\newblock In \emph{Proceedings of the 34th International Conference on Machine
  Learning-Volume 70}, pp.\  1558--1567. JMLR. org, 2017.

\bibitem[Jegou et~al.(2011)Jegou, Douze, and Schmid]{jegou2011product}
Jegou, H., Douze, M., and Schmid, C.
\newblock Product quantization for nearest neighbor search.
\newblock \emph{IEEE transactions on pattern analysis and machine
  intelligence}, 33\penalty0 (1):\penalty0 117--128, 2011.

\bibitem[Jeong \& Song(2018)Jeong and Song]{jeong2018efficient}
Jeong, Y. and Song, H.~O.
\newblock Efficient end-to-end learning for quantizable representations.
\newblock \emph{arXiv preprint arXiv:1805.05809}, 2018.

\bibitem[Kim et~al.(2018)Kim, Yoon, Dia, Kim, Bengio, and Ahn]{kim2018bayesian}
Kim, T., Yoon, J., Dia, O., Kim, S., Bengio, Y., and Ahn, S.
\newblock Bayesian model-agnostic meta-learning.
\newblock \emph{arXiv preprint arXiv:1806.03836}, 2018.

\bibitem[Kingma \& Ba(2014)Kingma and Ba]{kingma2014adam}
Kingma, D.~P. and Ba, J.
\newblock Adam: A method for stochastic optimization.
\newblock \emph{arXiv preprint arXiv:1412.6980}, 2014.

\bibitem[Koch et~al.(2015)Koch, Zemel, and Salakhutdinov]{koch2015siamese}
Koch, G., Zemel, R., and Salakhutdinov, R.
\newblock Siamese neural networks for one-shot image recognition.
\newblock In \emph{ICML deep learning workshop}, volume~2, 2015.

\bibitem[Krause et~al.(2013)Krause, Stark, Deng, and
  Fei-Fei]{KrauseStarkDengFei-Fei_3DRR2013}
Krause, J., Stark, M., Deng, J., and Fei-Fei, L.
\newblock 3d object representations for fine-grained categorization.
\newblock In \emph{4th International IEEE Workshop on 3D Representation and
  Recognition (3dRR-13)}, Sydney, Australia, 2013.

\bibitem[Krizhevsky(2009)]{Krizhevsky09learningmultiple}
Krizhevsky, A.
\newblock Learning multiple layers of features from tiny images.
\newblock Technical report, 2009.

\bibitem[Kuncheva \& Whitaker(2003)Kuncheva and Whitaker]{kuncheva2003measures}
Kuncheva, L.~I. and Whitaker, C.~J.
\newblock Measures of diversity in classifier ensembles and their relationship
  with the ensemble accuracy.
\newblock \emph{Machine learning}, 51\penalty0 (2):\penalty0 181--207, 2003.

\bibitem[Lee et~al.(2019)Lee, Maji, Ravichandran, and Soatto]{lee2019meta}
Lee, K., Maji, S., Ravichandran, A., and Soatto, S.
\newblock Meta-learning with differentiable convex optimization.
\newblock In \emph{Proceedings of the IEEE Conference on Computer Vision and
  Pattern Recognition}, pp.\  10657--10665, 2019.

\bibitem[Lee \& Choi(2018)Lee and Choi]{Lee2018}
Lee, Y. and Choi, S.
\newblock Gradient-based meta-learning with learned layerwise metric and
  subspace.
\newblock In \emph{Proceedings of the International Conference on Machine
  Learning}, 2018.

\bibitem[Liu et~al.(2018)Liu, Lee, Park, Kim, and Yang]{liu2018transductive}
Liu, Y., Lee, J., Park, M., Kim, S., and Yang, Y.
\newblock Transductive propagation network for few-shot learning.
\newblock \emph{arXiv preprint arXiv:1805.10002}, 2018.

\bibitem[Mishra et~al.(2017)Mishra, Rohaninejad, Chen, and
  Abbeel]{mishra2017simple}
Mishra, N., Rohaninejad, M., Chen, X., and Abbeel, P.
\newblock A simple neural attentive meta-learner.
\newblock \emph{arXiv preprint arXiv:1707.03141}, 2017.

\bibitem[Movshovitz-Attias et~al.(2017)Movshovitz-Attias, Toshev, Leung, Ioffe,
  and Singh]{movshovitz2017no}
Movshovitz-Attias, Y., Toshev, A., Leung, T.~K., Ioffe, S., and Singh, S.
\newblock No fuss distance metric learning using proxies.
\newblock In \emph{Proceedings of the IEEE International Conference on Computer
  Vision}, pp.\  360--368, 2017.

\bibitem[Munkhdalai et~al.(2017)Munkhdalai, Yuan, Mehri, and
  Trischler]{munkhdalai2017rapid}
Munkhdalai, T., Yuan, X., Mehri, S., and Trischler, A.
\newblock Rapid adaptation with conditionally shifted neurons.
\newblock \emph{arXiv preprint arXiv:1712.09926}, 2017.

\bibitem[Oord et~al.(2018)Oord, Li, and Vinyals]{oord2018representation}
Oord, A. v.~d., Li, Y., and Vinyals, O.
\newblock Representation learning with contrastive predictive coding.
\newblock \emph{arXiv preprint arXiv:1807.03748}, 2018.

\bibitem[Oreshkin et~al.(2018)Oreshkin, L{\'o}pez, and
  Lacoste]{oreshkin2018tadam}
Oreshkin, B., L{\'o}pez, P.~R., and Lacoste, A.
\newblock Tadam: Task dependent adaptive metric for improved few-shot learning.
\newblock In \emph{Advances in Neural Information Processing Systems}, pp.\
  721--731, 2018.

\bibitem[Ravi \& Beatson(2019)Ravi and Beatson]{ravi2018amortized}
Ravi, S. and Beatson, A.
\newblock Amortized bayesian meta-learning.
\newblock In \emph{International Conference on Learning Representations}, 2019.
\newblock URL \url{https://openreview.net/forum?id=rkgpy3C5tX}.

\bibitem[Ravi \& Larochelle(2016)Ravi and Larochelle]{ravi2016optimization}
Ravi, S. and Larochelle, H.
\newblock Optimization as a model for few-shot learning.
\newblock 2016.

\bibitem[Razavi et~al.(2019)Razavi, Oord, and Vinyals]{razavi2019generating}
Razavi, A., Oord, A. v.~d., and Vinyals, O.
\newblock Generating diverse high-fidelity images with vq-vae-2.
\newblock \emph{arXiv preprint arXiv:1906.00446}, 2019.

\bibitem[Rolfe(2016)]{rolfe2016discrete}
Rolfe, J.~T.
\newblock Discrete variational autoencoders.
\newblock \emph{arXiv preprint arXiv:1609.02200}, 2016.

\bibitem[Rusu et~al.(2018)Rusu, Rao, Sygnowski, Vinyals, Pascanu, Osindero, and
  Hadsell]{rusu2018meta}
Rusu, A.~A., Rao, D., Sygnowski, J., Vinyals, O., Pascanu, R., Osindero, S.,
  and Hadsell, R.
\newblock Meta-learning with latent embedding optimization.
\newblock \emph{arXiv preprint arXiv:1807.05960}, 2018.

\bibitem[Shamir et~al.(2010)Shamir, Sabato, and Tishby]{shamir2010learning}
Shamir, O., Sabato, S., and Tishby, N.
\newblock Learning and generalization with the information bottleneck.
\newblock \emph{Theoretical Computer Science}, 411\penalty0 (29-30):\penalty0
  2696--2711, 2010.

\bibitem[Shannon(1948)]{shannon1948mathematical}
Shannon, C.~E.
\newblock A mathematical theory of communication.
\newblock \emph{Bell system technical journal}, 27\penalty0 (3):\penalty0
  379--423, 1948.

\bibitem[Shwartz-Ziv \& Tishby(2017)Shwartz-Ziv and Tishby]{shwartz2017opening}
Shwartz-Ziv, R. and Tishby, N.
\newblock Opening the black box of deep neural networks via information.
\newblock \emph{arXiv preprint arXiv:1703.00810}, 2017.

\bibitem[Snell et~al.(2017)Snell, Swersky, and Zemel]{snell2017prototypical}
Snell, J., Swersky, K., and Zemel, R.
\newblock Prototypical networks for few-shot learning.
\newblock In \emph{Advances in Neural Information Processing Systems}, pp.\
  4077--4087, 2017.

\bibitem[Sohn(2016)]{sohn2016improved}
Sohn, K.
\newblock Improved deep metric learning with multi-class n-pair loss objective.
\newblock In \emph{Advances in Neural Information Processing Systems}, pp.\
  1857--1865, 2016.

\bibitem[Strouse \& Schwab(2017)Strouse and Schwab]{strouse2017deterministic}
Strouse, D. and Schwab, D.~J.
\newblock The deterministic information bottleneck.
\newblock \emph{Neural computation}, 29\penalty0 (6):\penalty0 1611--1630,
  2017.

\bibitem[Sung et~al.(2018)Sung, Yang, Zhang, Xiang, Torr, and
  Hospedales]{sung2018learning}
Sung, F., Yang, Y., Zhang, L., Xiang, T., Torr, P.~H., and Hospedales, T.~M.
\newblock Learning to compare: Relation network for few-shot learning.
\newblock In \emph{Proceedings of the IEEE Conference on Computer Vision and
  Pattern Recognition}, pp.\  1199--1208, 2018.

\bibitem[Tishby \& Zaslavsky(2015)Tishby and Zaslavsky]{tishby2015deep}
Tishby, N. and Zaslavsky, N.
\newblock Deep learning and the information bottleneck principle.
\newblock In \emph{2015 IEEE Information Theory Workshop (ITW)}, pp.\  1--5.
  IEEE, 2015.

\bibitem[Tishby et~al.(2000)Tishby, Pereira, and Bialek]{tishby2000information}
Tishby, N., Pereira, F.~C., and Bialek, W.
\newblock The information bottleneck method.
\newblock \emph{arXiv preprint physics/0004057}, 2000.

\bibitem[Triantafillou et~al.(2019)Triantafillou, Zhu, Dumoulin, Lamblin, Xu,
  Goroshin, Gelada, Swersky, Manzagol, and Larochelle]{triantafillou2019meta}
Triantafillou, E., Zhu, T., Dumoulin, V., Lamblin, P., Xu, K., Goroshin, R.,
  Gelada, C., Swersky, K., Manzagol, P.-A., and Larochelle, H.
\newblock Meta-dataset: A dataset of datasets for learning to learn from few
  examples.
\newblock \emph{arXiv preprint arXiv:1903.03096}, 2019.

\bibitem[van~den Oord et~al.(2017)van~den Oord, Vinyals, et~al.]{van2017neural}
van~den Oord, A., Vinyals, O., et~al.
\newblock Neural discrete representation learning.
\newblock In \emph{Advances in Neural Information Processing Systems}, pp.\
  6306--6315, 2017.

\bibitem[Vinyals et~al.(2016)Vinyals, Blundell, Lillicrap, Wierstra,
  et~al.]{vinyals2016matching}
Vinyals, O., Blundell, C., Lillicrap, T., Wierstra, D., et~al.
\newblock Matching networks for one shot learning.
\newblock In \emph{Advances in neural information processing systems}, pp.\
  3630--3638, 2016.

\bibitem[Wah et~al.(2011)Wah, Branson, Welinder, Perona, and
  Belongie]{wah2011caltech}
Wah, C., Branson, S., Welinder, P., Perona, P., and Belongie, S.
\newblock The caltech-ucsd birds-200-2011 dataset.
\newblock 2011.

\bibitem[Ye et~al.(2018)Ye, Hu, Zhan, and Sha]{ye2018learning}
Ye, H.-J., Hu, H., Zhan, D.-C., and Sha, F.
\newblock Learning embedding adaptation for few-shot learning.
\newblock \emph{arXiv preprint arXiv:1812.03664}, 2018.

\bibitem[Yi et~al.(2014)Yi, Lei, and Li]{yideep2014}
Yi, D., Lei, Z., and Li, S.
\newblock Deep metric learning for practical person re-identification.
\newblock \emph{ArXiv e-prints}, 2014.

\bibitem[Yin et~al.(2019)Yin, Tucker, Zhou, Levine, and Finn]{yin2019meta}
Yin, M., Tucker, G., Zhou, M., Levine, S., and Finn, C.
\newblock Meta-learning without memorization.
\newblock \emph{arXiv preprint arXiv:1912.03820}, 2019.

\bibitem[Zintgraf et~al.(2018)Zintgraf, Shiarlis, Kurin, Hofmann, and
  Whiteson]{zintgraf2018caml}
Zintgraf, L.~M., Shiarlis, K., Kurin, V., Hofmann, K., and Whiteson, S.
\newblock Caml: Fast context adaptation via meta-learning.
\newblock \emph{arXiv preprint arXiv:1810.03642}, 2018.

\end{thebibliography}

\onecolumn
\appendix

\icmltitle{Supplementary Material}

\section{Previous Loss functions Are Approximations to Mutual Information}
\label{app:losses}
\paragraph{Cross-entropy Loss} \label{subsec:xent}
The cross-entropy loss has directly been used for few-shot classification \citep{vinyals2016matching,snell2017prototypical}.

Let $q(\y|\tx ;\phi)$ be a parameterized prediction of $\y$ given $\tx$, which tries to approximate the true conditional distribution $q(\y|\tx)$.
Typically in a classification network, $\phi$ is the parameters of a learned projection matrix and $q(\cdot)$ is the final linear layer.
The expected cross-entropy loss can be written as
\be
\mathrm{xent}(Y, \tX) = \E_{\y \sim Y, \tx \sim \tX} \left[ -\log q(\y|\tx, \phi) \right].
\label{eq:xent}
\ee
Assuming that the approximate distribution $q(\cdot)$ is sufficiently close to $p(\y|\tx)$, 
minimizing (\ref{eq:xent}) can be seen as
\be
\argmin \mathrm{xent}(Y, \tX) 
&\approx& \argmin \E_{\y \sim Y, \tx \sim \tX} \left[ -\log p(\y|\tx) \right] \\
&=& \argmin H(Y|\tX)
= \argmax I(\tX; Y),
\ee
where the last equality uses the fact that $H(Y)$ is independent of model parameters.
Therefore, cross-entropy minimization is approximate maximization of the mutual information between representation $\tX$ and labels $Y$.

The approximation is that we parameterized $q(\y|\tx;\phi)$ as a linear projection.
This structure cannot generalize to new classes because the parameters $\phi$ are specific to the labels $\y$ seen during training.
For a model to generalize to unseen classes, one must amortize the learning of this approximate conditional distribution.
\citep{vinyals2016matching,snell2017prototypical} sidestepped this issue by using the embeddings for each class as $\phi$.

\paragraph{Triplet Loss}
The Triplet loss \citep{hoffer2015deep} is defined as
\be
\triloss = \norm{\tx_q-\tx_p}_2^2 - \norm{\tx_q-\tx_n}_2^2,
\ee
where $\tx_q, \tx_p, \tx_n \in \mathbb{R}^d$ are the embedding vectors of query, positive, and negative images.
Let $\y_q$ denote the label of the query data.
Recall that the pdf function of a unit Gaussian is
$
\log N(\tx|\mu, 1) 
= -c_1 - c_2\norm{\tx-\mu}_2^2, 
$
where $c_1,c_2$ are constants.
Let $p_p(\tx)=N(\tx_p, 1)$ and $p_n(\tx)=N(\tx_n, 1)$ be unit Gaussian distributions centered at $\tx_p, \tx_n$ respectively.
We have
\be
\E \left[ -\triloss \right]
&\propto& \E \left[ \log p_p(\tx) - \log p_n(\tx)\right] \\
&\approx& \E \left[ \log p_p(\tx) - \log p(\tx)\right] \\
&=& -H(\tX|Y) + H(\tX) = I(\tX;Y).
\ee
Two approximations were made in the process.
We first assumed that the embedding distribution of images not in $\y_q$ is equal to the distribution of all embeddings.
This is reasonable when each class only represents a small fraction of the full data.
We also approximated the embedding distributions $p(\tx|\y), p(\tx)$ with unit Gaussian distributions centered at single samples from each.

\paragraph{N-pair Loss}
Multiclass $N$-pair loss \citep{sohn2016improved} was proposed as an alternative to Triplet loss.
This loss function requires one positive embedding $\tx^+$ and multiple negative embeddings $\tx_1,$ $\ldots,$ $\tx_{N-1}$,
and takes the form
\be 
- \log \frac{\exp (\tx^\top \tx^+)}{\exp(\tx^\top \tx^+) + \sum_{i=1}^{N-1} \exp(\tx^\top \tx_i)}.
\ee
This can be seen as the cross-entropy loss applied to $\textrm{softmax}(\tx^\top \tx^+, \tx^\top \tx_1, \ldots, \tx^\top \tx_{N-1})$.

Following the same logic as the cross-entropy loss, this is also an approximation to $I(\tX; Y)$.
This objective should have less variance than Triplet loss since it approximates $p(\tx)$ using more examples.

\paragraph{Adversarial Metric Learning}
Deep Adversarial Metric Learning \citep{duan2018deep} tackles the problem of most negative exmples being uninformative by directly generating meaningful negative embeddings.
This model employs a \textit{generator} which takes as input the embeddings of anchor, positive, and negative images.
The generator then outputs a "synthetic negative" embedding that is hard to distinguish from a positive embedding while being close to the negative embedding.

This can be seen as optimizing
\be 
\E \left[ \log p_p(\tx) - \log p(\tx)\right]
= I(\tX; Y)
\ee
by estimating $p(\tx)$ using a generative network rather than directly from samples.
Rather than modelling the marginal distribution $p(\tx)$, this method conditionally models $p(\tx;\tx_q,\tx_p,\tx_n)$ so that $\tx$ is hard to distinguish from $\tx_p$ while sufficiently close to both $\tx_q$ and $\tx_n$.

\section{Proof of Theorem 1}
\label{app:analysis}
\setcounter{theorem}{0}

We restate and prove our main theorem.

\begin{theorem}
Let $d_\Theta$ be the VC dimension of the encoder $\tX(\cdot)$.
Let $\hat{I}(\tX(X_T, \theta); Y_T)$ be the empirical estimate of the mutual information using finite dataset $D_T$, 
and define empirical loss as
\begin{eqnarray}
\hat{\calL}(T^{1:n}, \theta) = - \frac{1}{n} \sum_{i=1}^n \hat{I}(\tX(X_{T^i}, \theta); Y_{T^i}).
\end{eqnarray}
The following inequality holds with high probability:
\begin{eqnarray}
\calL(\tau, \theta) - \hat{\calL} (T^{1:n}, \theta) 
\leq
O \left( \sqrt{\frac{d_\Theta}{n} \log \frac{n}{d_\Theta}} \right) + 
O\left( \frac{|\tX| \log(m)}{\sqrt{m}} \right) + O\left( \frac{|\tX| |Y|}{m} \right)
\label{eq:bound}
\end{eqnarray}
\end{theorem}
\begin{proof}

We use the following lemma from \cite{shamir2010learning}, which we restate using our notation.
\begin{lemma}
Let $\tX$ be a random mapping of $X$.
Let $D$ be a sample of size $m$ drawn from the joint probability distribution $p(X, Y)$.
Denote the empirical mutual information observed from $D$ between $\tX$ and $Y$ as $\hat{I}(\tX; Y)$.
For any $\delta \in (0, 1)$, the following holds with probability at least $1 - \delta$:
\begin{eqnarray}
|I(\tX; Y)-\hat{I}(\tX; Y)|
\leq \frac{(3|\tX|+2) \log (m) \sqrt{\log (4 / \delta)}}{\sqrt{2 m}}
+ \frac{(|Y|+1)(|\tX|+1)-4}{m}
\end{eqnarray}
\end{lemma}

We simplify this and plug in our specific quantities of interest ($\tX(X_T, \theta)$, $Y_T$):
\begin{eqnarray}
\label{eq:mi_bound}
\abs{ I\left(\tX(X_T, \theta); Y_T \right) - \hat{I}\left( \tX(X_T, \theta); Y_T \right)}
\leq O\left( \frac{|\tX| \log(m)}{\sqrt{m}} \right) + O\left( \frac{|\tX| |Y|}{m} \right).
\end{eqnarray}

We similarly bound the error caused by estimating $\calL$ with a finite number of tasks sampled from $\tau$.
Denote the finite sample estimate of $\calL$ as
\begin{eqnarray}
\hat{\calL}(\tau, \theta) 
= - \frac{1}{n} \sum_{i=1}^n I(\tX(X_{T^i}, \theta); Y_{T^i}).
\end{eqnarray}

Let the mapping $X \mapsto \tX$ be parameterized by $\theta \in \Theta$ and let this model have VC dimension $d_\Theta$.
Using $d_\Theta$, we can state that with high probability,
\begin{eqnarray}
\label{eq:task_bound}
\abs{ \calL(\tau, \theta) - \hat{\calL}(\tau, \theta) }
\leq O \left( \sqrt{\frac{d_\Theta}{n} \log \frac{n}{d_\Theta}} \right),
\end{eqnarray}
where $d_\Theta$ is the VC dimension of hypothesis class $\Theta$.

Combining equations (\ref{eq:task_bound}, \ref{eq:mi_bound}), we have with high probability
\begin{align}
\bigg| \calL(\tau, \theta) &- \left(- \frac{1}{n} \sum_{i=1}^n \hat{I}(\tX(X_{T^i}, \theta); Y_{T^i}) \right) \bigg| \\
\leq &
\abs{ \calL(\tau, \theta) - \hat{\calL}(\tau, \theta) }
+ O\left( \frac{|\tX| \log(m)}{\sqrt{m}} \right) + O\left( \frac{|\tX| |Y|}{m} \right) \\
\leq &
O \left( \sqrt{\frac{d_\Theta}{n} \log \frac{n}{d_\Theta}} \right) + 
O\left( \frac{|\tX| \log(m)}{\sqrt{m}} \right) + O\left( \frac{|\tX| |Y|}{m} \right)
\end{align}
\end{proof}
\newpage
\begin{figure}
	\centering
	\includegraphics[width=\linewidth]{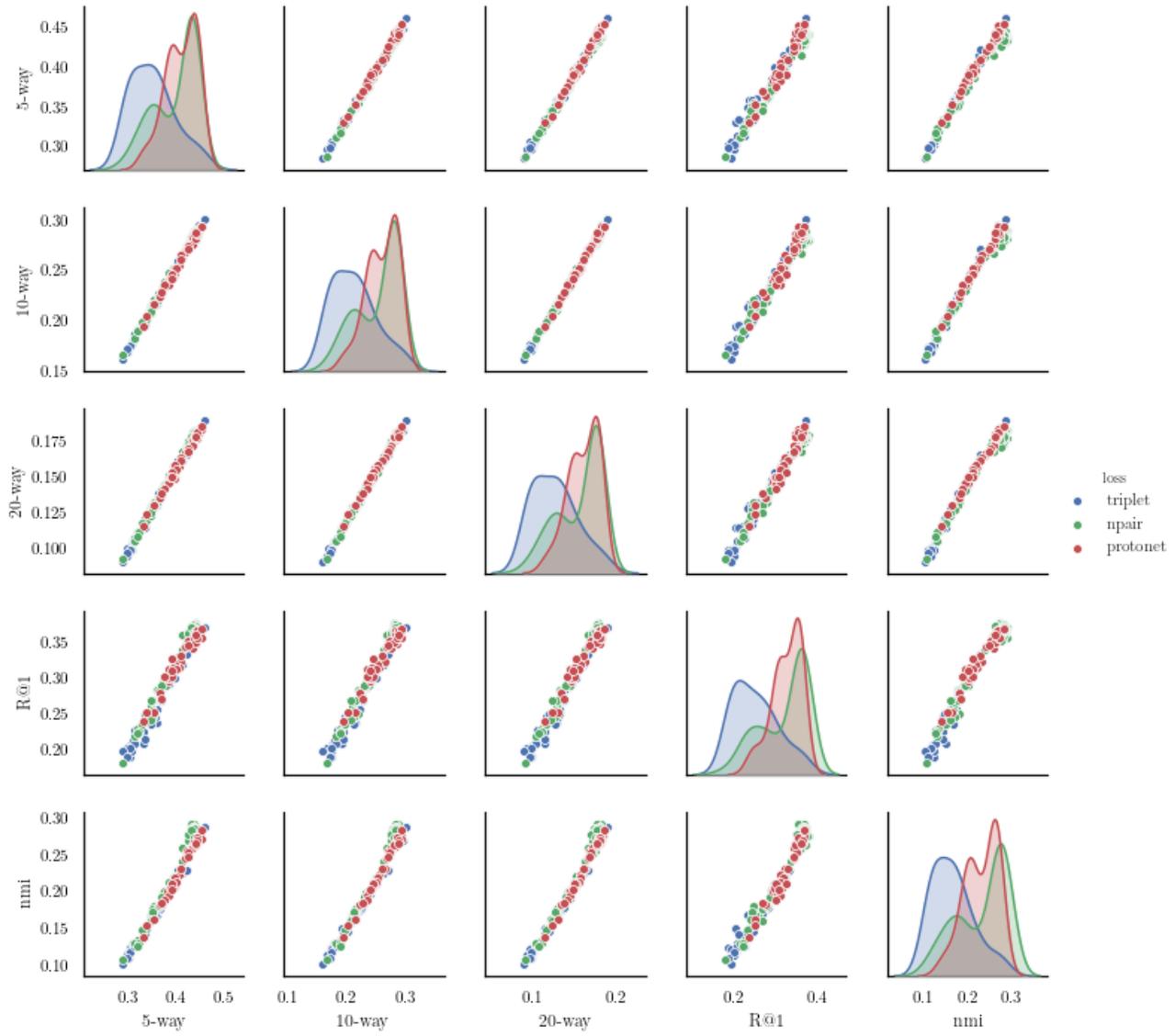}
	\caption{Correlation between few-shot accuracy and retrieval measures.}
	\label{fig:corr_all}
\end{figure}

\begin{figure}[t!]
\centering
\begin{subfigure}[t]{0.4\linewidth}
    \centering \includegraphics[width=\linewidth]{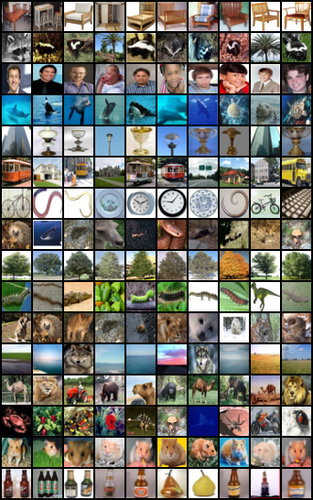} 
\end{subfigure}
\hspace{10pt}
\begin{subfigure}[t]{0.4\linewidth}
    \centering \includegraphics[width=\linewidth]{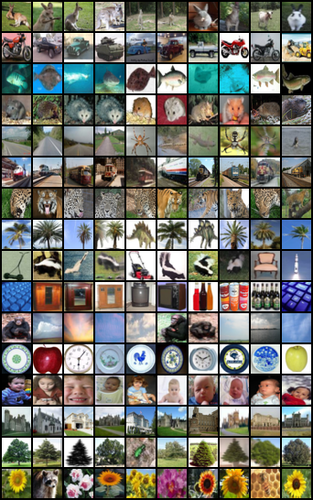} 
\end{subfigure} \\
\vspace{10pt}
\begin{subfigure}[t]{0.4\linewidth}
    \centering \includegraphics[width=\linewidth]{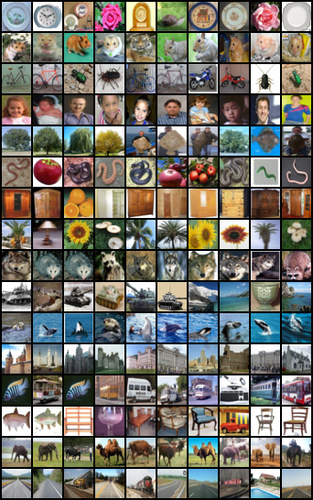} 
\end{subfigure}
\hspace{10pt}
\begin{subfigure}[t]{0.4\linewidth}
    \centering \includegraphics[width=\linewidth]{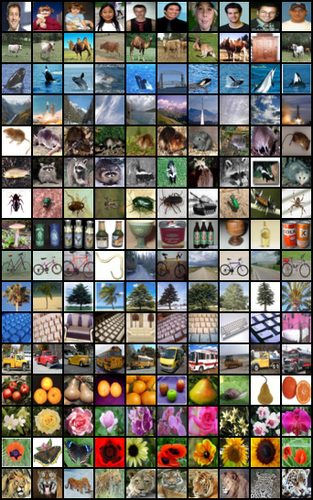} 
\end{subfigure}
\caption{
    Visualization of codes of a DIMCO model ($k=16$, $d=4$) trained on CIFAR100.
    For each of the $kd=64$ codewords, we show the top $10$ images from the \textit{test set} that assign it the highest marginal probability in each row.
}
\label{fig:code_vis_app}
\end{figure}

\begin{table*}[t]
\centering \begin{tabular}{
    ll
    cccc
    cccc
    }
\toprule
    & & \multicolumn{4}{c}{CIFAR-10}
    & \multicolumn{4}{c}{CIFAR-100} \\
    \cmidrule(l{4pt}r{4pt}){3-6}
    \cmidrule(l{4pt}r{4pt}){7-10}

    &
    & $d=2$ & $d=4$ & $d=8$ & $d=16$
    & $d=2$ & $d=4$ & $d=8$ & $d=16$ \\
\midrule
Product Quantization
    & {$k=2$}
    & 21.88 & 50.38 & 78.68 & 87.24
    & 2.84  & 7.77 & 15.24 & 32.47
    \\ & {$k=4$}
    & 60.70 & 86.86 & 89.53 & 90.92
    & 7.69  & 18.78 & 35.62 & 52.02
    \\ & {$k=8$}
    & 90.46 & 90.44 & 90.98 & 91.49
    & 15.28 & 31.04 & 50.74 & 58.82
    \\ & {$k=16$}
    & 91.29 & 91.02 & 91.27 & 91.52
    & 25.71 & 43.16 & 55.88 & 61.00
    \\
\midrule
DIMCO (ours)
    & {$k=2$}
    & 36.37 & 64.43 & 83.10 & 88.85
    & 3.66  & 10.03 & 19.31 & 33.46
    \\ & {$k=4$}
    & 62.92 & 88.47 & 90.78 & 91.04
    & 11.83 & 25.7  & 37.4  & 53.36
    \\ & {$k=8$}
    & 90.77 & 91.29 & 91.41 & 91.45
    & 31.9  & 46.22 & 52.77 & 58.84
    \\ & {$k=16$}
    & \bf 91.49 & \bf 91.68 & \bf 91.46 & \bf 91.57
    & \bf 46.17 & \bf 57.83 & \bf 61.11 & \bf 62.49
    \\
\bottomrule
\end{tabular}
\caption{
    Top-1 accuracies under various compressed embedding sizes ($k, d$).
    The original embeddings were continuous $64$-dimensional embeddings, which corresponds to $k=2^{32}, d=64$.
}
\label{tab:comp_cifar}
\end{table*}

\begin{table}[t]
    \centering \begin{tabular}{
        l
        cccc
        cccc
        }
    \toprule
        & \multicolumn{4}{c}{Product Quantization}
        & \multicolumn{4}{c}{DIMCO} \\
        \cmidrule(l{4pt}r{4pt}){2-5}
        \cmidrule(l{4pt}r{4pt}){6-9}
    
        & $d=2$ & $d=4$ & $d=8$ & $d=16$
        & $d=2$ & $d=4$ & $d=8$ & $d=16$ \\
    \midrule
        {$k=2$}
        & 0.21 & 0.31 & 1.00 & 5.21
        & 0.14  & 0.42 & 1.47 & 5.50
        \\ {$k=4$}
        & 0.20 & 0.92 & 4.86 & 21.45
        & 0.67  & 2.68 & 11.80 & 26.96
        \\ {$k=8$}
        & 0.82 & 2.57 & 11.91 & 35.43
        & 1.71 & 12.82 & 33.28 & 48.11
        \\ {$k=16$}
        & 1.99 & 5.66 & 20.34 & 44.93
        & 4.90 & 26.22 & 44.01 & 55.34
        \\
    \bottomrule
    \end{tabular}
    \caption{
        Top-1 accuracies of compressed embeddings under various sizes ($k, d$) on the ImageNet dataset.
        The original embeddings were continuous $2048$-dimensional embeddings, which corresponds to $k=2^{32}, d=2048$.
    }
    \label{tab:comp_cifar_app}
    \end{table}

\begin{figure} \centering
    \includegraphics[width=.47\linewidth]{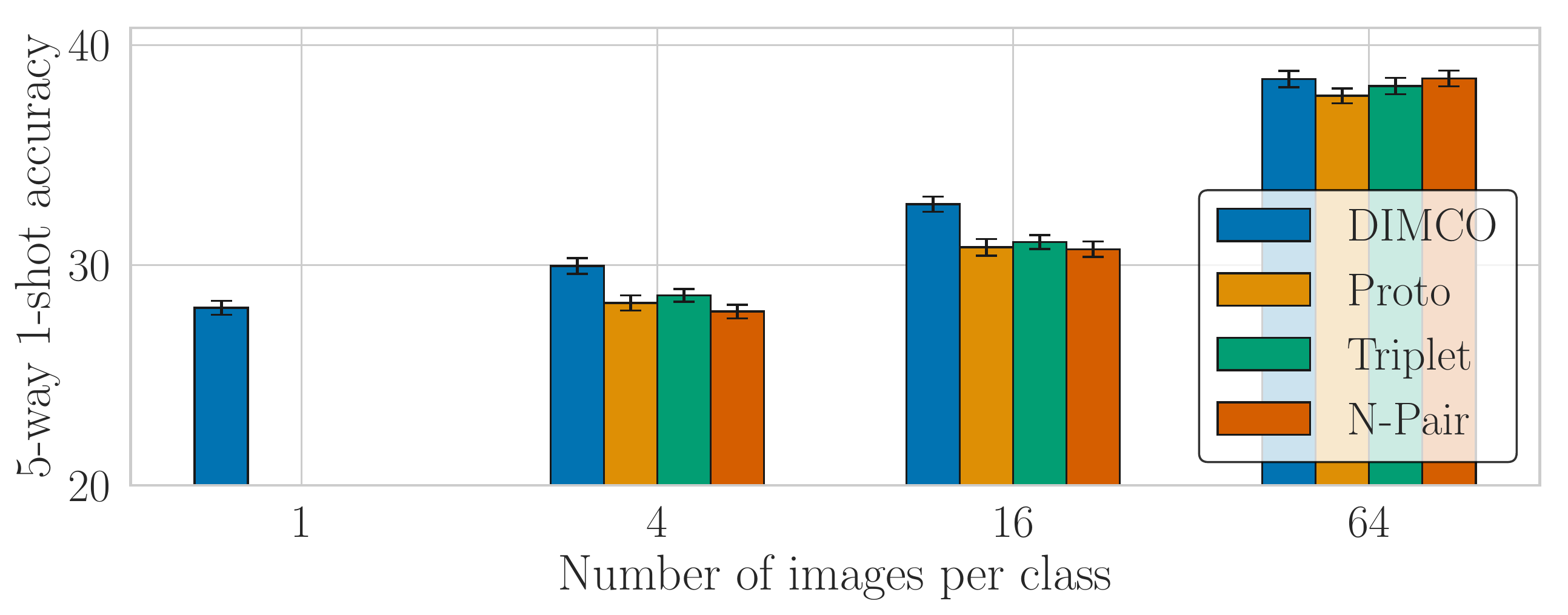} \hfill
    \includegraphics[width=.47\linewidth]{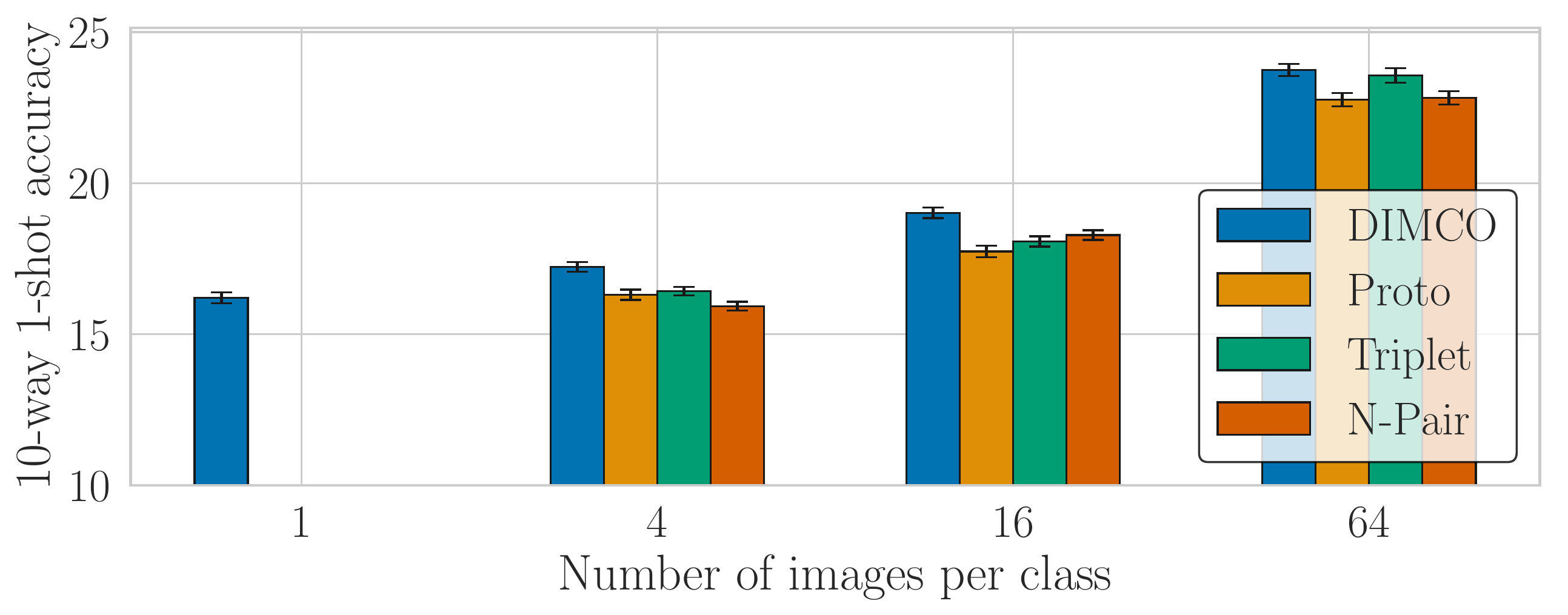} \\
    \includegraphics[width=.47\linewidth]{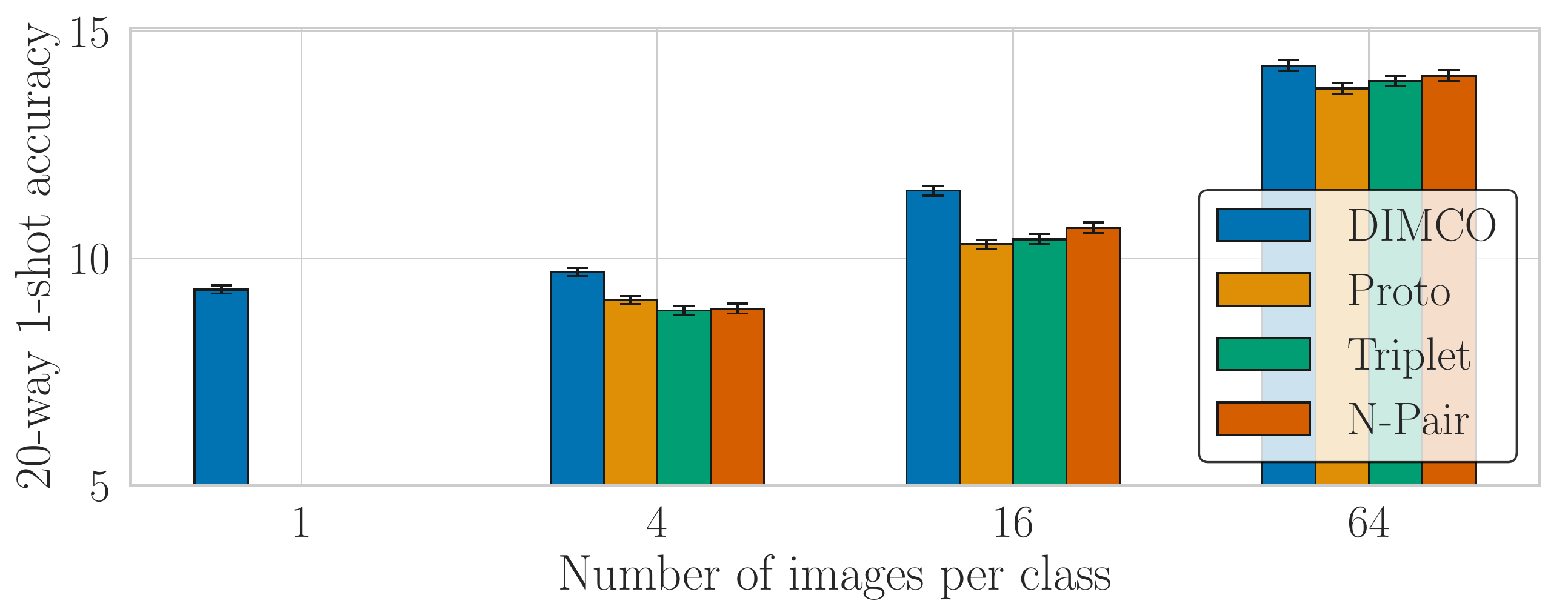} \hfill
    \includegraphics[width=.47\linewidth]{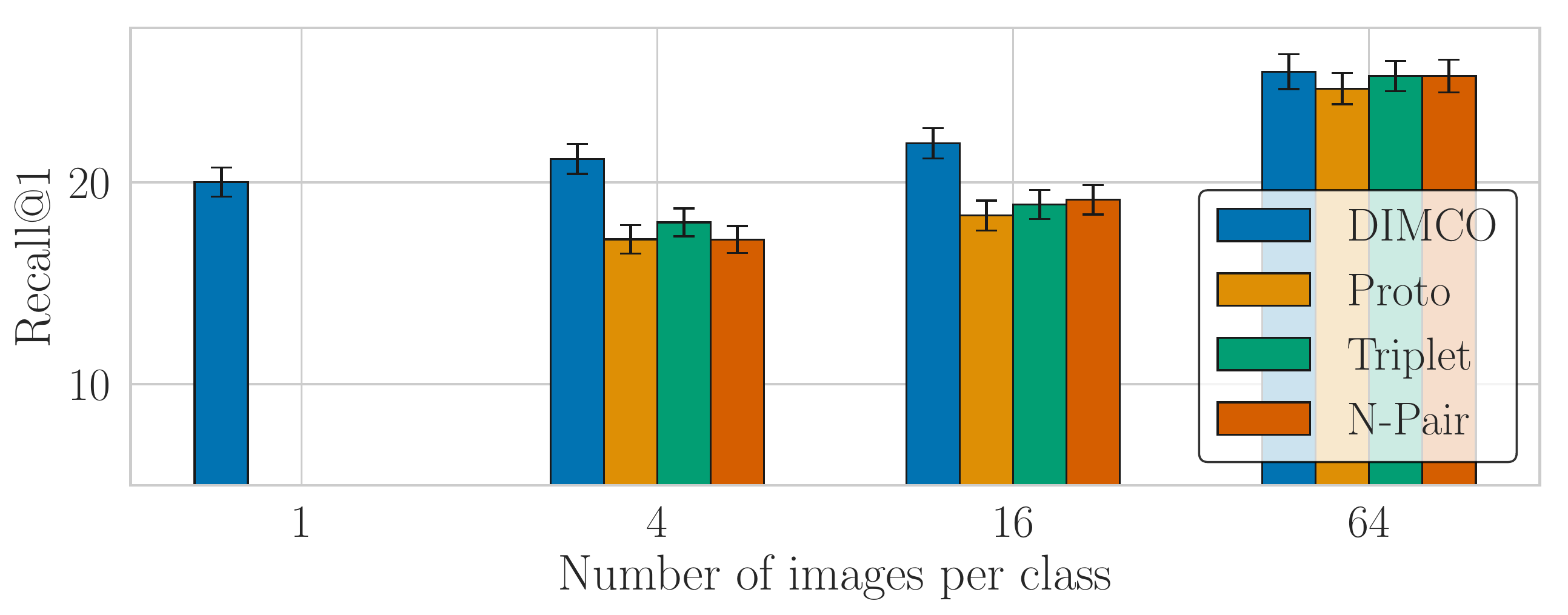}
\caption{
    Performance of methods trained using subsets of miniImageNet of varying size.
    The lowermost y axis value for each metric corresponds to the expected performance of random guessing.
    We show the mean and standard deviation of top $5$ runs from a hyperparameter sweep of $64$ runs per configuration.
}
\label{fig:small_app}
\end{figure}
\begin{table}[t]
\vspace{-7pt}
\caption{
    Few-shot classification accuracies on the miniImageNet benchmark.
    Grouped according to backbone architecture.
    $\dagger$ denotes transductive methods, which are more expressive by taking unlabeled examples into account.
}
\label{tab:fewshot_app}
\vspace{3pt}
\centering \begin{tabular}{lcc}
\toprule                    %
    Method                  & $5$-way $1$-shot      & $5$-way $5$-shot      \\
\midrule 
\multicolumn{2}{l}{\quad ConvNet (64-64-64-64)}                    \\
    $^\dagger$ TPN \citep{liu2018transductive}
                            & 55.51 $\pm$ 0.86      & 69.86 $\pm$ 0.65      \\
    $^\dagger$ FEAT \citep{ye2018learning}
                            & 55.75 $\pm$ 0.20      & 72.17 $\pm$ 0.16      \\
    MetaLSTM \citep{ravi2016optimization}             
                            & 43.44 $\pm$ 0.77      & 60.60 $\pm$ 0.71      \\
    MatchingNet \citep{vinyals2016matching}             
                            & 43.56 $\pm$ 0.84      & 55.31 $\pm$ 0.73      \\
    ProtoNet \citep{snell2017prototypical}               
                            & 49.42 $\pm$ 0.78      & 68.20 $\pm$ 0.66      \\
    RelationNet \citep{sung2018learning}               
                            & 50.44 $\pm$ 0.82      & 65.32 $\pm$ 0.70      \\
    R2D2 \citep{bertinetto2018meta}               
                            & 51.2 $\pm$ 0.6      & \BF 68.8 $\pm$ 0.1      \\
    MetaOptNet-SVM \citep{lee2019meta}          
                            & 52.87 $\pm$ 0.57      & 68.76 $\pm$ 0.48  \\
    DIMCO ($64, 64$)        & 47.33 $\pm$ 0.46      & 61.59 $\pm$ 0.52      \\
    DIMCO ($64, 128$)       & \BF 53.29 $\pm$ 0.47  & 64.79 $\pm$ 0.57      \\
\midrule \multicolumn{3}{l}{\quad ResNet-12} \\
    $^\dagger$ TPN \citep{liu2018transductive}
                            & 59.46                 & 75.65                 \\
    $^\dagger$ FEAT \citep{ye2018learning}
                            & 62.60 $\pm$ 0.20      & 78.06 $\pm$ 0.15  \\
    SNAIL \citep{mishra2017simple}                  
                            & 55.71 $\pm$ 0.99      & 68.88 $\pm$ 0.92      \\
    AdaResNet \citep{munkhdalai2017rapid}              
                            & 56.88 $\pm$ 0.62      & 71.94 $\pm$ 0.57      \\
    TADAM \citep{oreshkin2018tadam}                  
                            & 58.50 $\pm$ 0.30      & 76.70 $\pm$ 0.30      \\
    MetaOptNet-SVM \citep{lee2019meta}         
                            & \BF 62.64 $\pm$ 0.61  & \BF 78.63 $\pm$ 0.46  \\
    DIMCO ($128, 64$)       & 54.57 $\pm$ 0.47      & 65.45 $\pm$ 0.31      \\
    DIMCO ($128, 128$)      & 57.24 $\pm$ 0.44      & 69.31 $\pm$ 0.38      \\
\bottomrule
\end{tabular}
\end{table}

\section{Experiments}
\paragraph{Parameterizing the Code Layer}
Recall that each discrete code is parameterized by a $\Real^{k \times d}$ matrix.
A problem with a naive implementation of DIMCO is that simply using a linear layer that maps $\Real^{D}$ to $\Real^{k \times d}$ takes $Dkd$ parameters in that single layer.
This can be prohibitively expensive for large embeddings, e.g. $d=4096$.
We therefore parameterize this code layer as the product of two matrices, which sequentially map $\Real^D \rightarrow \Real^n \rightarrow \Real^{k \times d}$.
The total number of parameters required for this is $nD + nkd$.
We fix all $n=128$.
While more complicated tricks could reduce the parameter count even further, we found that this simple structure was sufficient to produce the results in this paper.

\paragraph{Correlation of Metrics}
We collect statistics from $8$ different independent runs,
and report the average of $500$ batches of $1$-shot accuracies, Recall@1, and mutual information.
$I(\tX; Y)$ was computed using balanced batches of $16$ images each from $5$ different classes.
In addition to the experiment in the paper, we measured the correlation between $1$-shot accuracies, $\textrm{Recall}@1$, and NMI using three previously proposed losses (triplet, npair, protonet).
\cref{fig:corr_all} shows that even for other methods for which mutual information is not the objective, mutual information strongly correlates with all other previous metrics.

\paragraph{Code Visualizations}
We provide additional visualizations of codes in \cref{fig:code_vis_app}.
These examples consistently show that each code encodes a semantic concept, and that such concepts can be but are not necessarily tied to a particular class.

\paragraph{Label-aware Bit Compression}
We computed top-1 accuracies using a kNN classifier on each type of embedding with $k=200$.
We present extended results comparing PQ and DIMCO on ImageNet embeddings in \cref{tab:comp_cifar_app}.
For CIFAR-10 and CIFAR-100 pretrained ResNet20, we used pretrained weights of open-sourced repository\footnote{\url{https://github.com/chenyaofo/CIFAR-pretrained-models}}, and for ImageNet pretrained ResNet50, we used \texttt{torchvision}.
We optimized the probablistic encoder with \texttt{Adam} optimizer \citep{kingma2014adam} with learning rate of $\textrm{1e-2}$ for CIFAR-100 and ImageNet, and $\textrm{3e-3}$ for CIFAR-10.

\paragraph{Small Train Set}

We performed an experiment to see how well DIMCO can generalize to new datasets \textit{after training with a small number of datasets}.
We trained each model using $\{1, 4, 16, 64\}$ samples from each class in the miniImageNet dataset.
For example, $4$ samples means that we train on ($64$ classes$\times$$4$ images) instead of the full ($64$ classes$\times$$600$ images).
We compare against three methods which use continuous embeddings for each datapoint: 
Triplet Nets \citep{hoffer2015deep},
multiclass N-pair loss \citep{sohn2016improved},
and ProtoNets \citep{snell2017prototypical}.
We show the $\textrm{Recall}@1$ metric, and additionally show $\{5, 10, 20\}$-way $1$-shot accuracies in the supplementary material.

\Cref{fig:small_app} shows that DIMCO learns much more effectively compared to previous methods when the number of examples per class is low.
We attribute this to DIMCO's tight generalization gap.
Since DIMCO uses fewer bits to describe each datapoint, the codes act as an implicit regularizer that helps generalization to unseen datasets.
We additionally note that DIMCO is the only method in \cref{fig:small_app} that can train using a dataset consisting of $1$ example per class.
DIMCO has this capability because, unlike other methods, DIMCO requires no support/query (also called train/test) split and maximizes the mutual information within a given batch.
In contrast, other methods require at least one support and one query example per class within each batch.

For this experiment, we used the Adam optimizer and performed a log-uniform hyperparameter sweep for learning rate $\in [\textrm{1e-7}, \textrm{1e-3}]$
For DIMCO, we swept $p \in [32, 128]$ and $d \in [16, 32]$.
For other methods, we made the embedding dimension $\in [16, 32]$.
For each combination of loss and number of training examples per class, we ran the experiment $64$ times and reported the mean and standard deviation of the top $5$.

\paragraph{Few-shot Classification}
For this experiment, we built on the code released by \citet{lee2019meta} (\url{https://github.com/kjunelee/MetaOptNet}) with minimal adjustments.
We used the repository's default datasets, augmentation, optimizer, and backbones.
The only difference was our added module for outputting discrete codes.
We show an extended table with citations in \cref{tab:fewshot_app}.

\end{document}